%% file: icdm_2026.tex
\documentclass[conference]{IEEEtran}
\IEEEoverridecommandlockouts

\usepackage[numbers,sort&compress]{natbib}

\usepackage[utf8]{inputenc} % allow utf-8 input
\usepackage[T1]{fontenc}    % use 8-bit T1 fonts
\usepackage{hyperref}       % hyperlinks
\usepackage{url}            % simple URL typesetting
\usepackage{booktabs}       % professional-quality tables
\usepackage{amsfonts}       % blackboard math symbols
\usepackage{nicefrac}       % compact symbols for 1/2, etc.
\usepackage{microtype}      % microtypography
\usepackage{xcolor}         % colors
\usepackage{tikz}           % circled table cells (avg.\ rank)
\newcommand{\circledavg}[1]{%
  \tikz[baseline=(n.base)]{%
    \node[draw,circle,inner sep=0.5pt,minimum size=1.45em](n){\centering #1};%
  }%
}
\usepackage{amsmath}
\usepackage{amsthm}
\usepackage[ruled,vlined]{algorithm2e}
\usepackage{graphicx}
\usepackage[font=footnotesize]{subcaption}
\usepackage{tikz-cd}
\usepackage{booktabs}
\usepackage{wasysym}

\usepackage[table]{xcolor}

\usepackage{pifont} % for symbols

\newcommand{\poscell}{\cellcolor{green!20}\ding{51}} % checkmark
\newcommand{\midcell}{\cellcolor{yellow!30}\ding{51}\ding{55}} % neutral / partial
\newcommand{\negcell}{\cellcolor{red!20}\ding{55}} % cross

\makeatletter
\newcommand{\linebreakand}{%
  \end{@IEEEauthorhalign}
  \hfill\mbox{}\par
  \mbox{}\hfill\begin{@IEEEauthorhalign}}
\makeatother

\newtheorem{proposition}{Proposition}
\newtheorem{theorem}{Theorem}

\DeclareMathOperator*{\argmin}{argmin}
\DeclareMathOperator*{\argmax}{argmax}
\newcommand{\KL}[2]{\operatorname{KL}\!\left(#1 \,\middle\|\, #2\right)}
\newcommand{\xobs}{x_m}
\newcommand{\xmis}{\overline{x}_m}
\newcommand{\ourname}{Impute-EM}

\title{\ourname: Native Mixed-State Diffusion Models for Heterogeneous Data Imputation}

\author{\IEEEauthorblockN{Sergei Kholkin}
\IEEEauthorblockA{Applied AI Institute}
\and
\IEEEauthorblockN{Kirill Sokolov}
\IEEEauthorblockA{Applied AI Institute, Moscow State University}
\linebreakand
\IEEEauthorblockN{Dmitry Baranchuk}
\IEEEauthorblockA{Yandex Research}
\and
\IEEEauthorblockN{Evgeny Burnaev}
\IEEEauthorblockA{Applied AI Institute, AXXX}
\and
\IEEEauthorblockN{Alexander Korotin}
\IEEEauthorblockA{Applied AI Institute, AXXX}}

\begin{document}

\maketitle

\begin{abstract}
Missing values are ubiquitous in heterogeneous data mining, where numerical, categorical, and binary variables often coexist. Many imputation methods, especially diffusion-based ones, treat discrete variables through continuous surrogates such as one-hot relaxations rather than modeling them natively. This creates a mismatch between the model state space and the mixed discrete and continuous structure of the data. We propose \ourname, an Expectation Maximization style framework that alternates between imputing missing entries with the current model and refitting a diffusion backbone on completed data. We instantiate \ourname\ with native mixed-state diffusion backbones for heterogeneous data, combining Gaussian and masked categorical components without one-hot relaxations. In exact settings, we characterize the update and show that the observed mask-indexed marginals match the targets at the limit, while making explicit that the full data distribution is generally non-identifiable from incomplete observations alone. Empirically, \ourname\ delivers the best distributional fidelity on mixed-type tabular imputation, on which downstream modeling relies, with text imputation serving as a controlled validation of the native discrete backbone.
\end{abstract}

\vspace{-2mm}
\section{Introduction}
\vspace{-2mm}

Missing data is ubiquitous, and reliable imputation is a central step in data-driven systems: imputed tables are often reused for downstream modeling and analysis, so preserving the marginal and dependency structure of the data matters alongside recovering individual entries. This is especially important for heterogeneous data, where numerical, categorical, and binary variables often coexist across tabular data \citep{du2024remasker}, audio data \citep{dror2025token}, and graph data \citep{you2020handling}.

% \textcolor{red}{explain somewhere here the difference between imputatio (learning from missig entries) and learning from full data and imputing values}.

Prior work on data imputation spans classical statistical methods, conventional machine learning, and modern deep models. Early approaches include nearest-neighbor imputation \cite{pujianto2019knn} and parametric density models such as Gaussian mixtures \citep{garcia2010pattern}. Recent predictive methods learn missing entries from observed entries using masking or graph-based structure \citep{du2024remasker,you2020handling,zhong2023igrm}. Generative methods instead model the joint distribution of observed and missing entries and impute by conditional sampling \citep{yoon2018gain,mattei2019miwae,richardson2020mcflow,ouyang2023missdiff,zheng2022tabcsdi}. These methods are attractive for heterogeneous imputation because they can represent uncertainty over plausible completions rather than returning only a single deterministic fill-in.

Despite the success of modern generative models, including diffusion models \citep{ho2020denoising}, training generative imputers from incomplete data remains challenging. A common approach is to optimize incomplete-data objectives that evaluate likelihood or reconstruction losses only on observed entries \citep{ouyang2023missdiff,mattei2019miwae}. These objectives are practical, but they do not directly specify how the unobserved coordinates should be completed during training. The \textbf{Expectation Maximization algorithm} provides a natural alternative by alternating between imputing missing values under the current model and updating the model on the resulting completed data \citep{diffputer,dempster1977em,yu2025missing_miri,yoon2018gain}. However, its use with native discrete and mixed-state diffusion backbones remains underexplored.

% \textbf{rarely provides} optimized objectives, \textbf{theoretical guarantees}, or fixed point characterization.

Moreover, recent diffusion-based imputation methods mostly rely on continuous diffusion models \citep{ouyang2023missdiff,diffputer}, often representing categorical variables as one-hot vectors or continuous relaxations. This creates a mismatch between the model state space and the native structure of discrete variables, and it can become inefficient for high-dimensional categorical data \citep{sahoo2024mdlm}.
% In that regard the natural state space specific data imputation problem solutions on discrete and combinations of discrete and continuous random variables, which we call general state spaces, remain underexplored.
Native discrete diffusion models and mixed continuous--categorical diffusion models are therefore a natural fit for heterogeneous imputation, but their use in this setting remains limited \citep{du2024remasker,dror2025token,you2020handling}. 

% The core challenge is the incomplete likelihood. A generative model trained directly on incomplete samples must estimate a full data distribution from partially observed values, and its samples during the generation process are not guaranteed to agree with the observed entries due to training by the utilization of the incomplete likelihood. 

% The core challenge is the incomplete likelihood: learning a generative model requires the full data distribution, while the missing coordinates are precisely the unknown quantities. The classical EM algorithm provides a principled route for such problems by alternating between imputing missing values and updating the model \citep{dempster1977em}. However, combining EM with powerful diffusion models is still challenging, since diffusion models naturally generate complete samples and conditional inference over missing entries is not direct.

% In this work, we revisit EM style iterative imputation through a principled objective that uses all information available in the training dataset and yields theoretical guarantees for imputation quality. We further use general state space diffusion models as generative backbones for \ourname, enabling imputation in continuous, discrete, and mixed state spaces.

% In this work, we revisit EM style iterative imputation through the lens of general state space diffusion models. Our goal is to use the expressive power of diffusion models while preserving a principled iterative procedure that can handle continuous, discrete, and mixed data.

Our contributions are:

\begin{itemize}
    \item We formulate missing-data learning as matching mask-indexed observed marginals and derive the corresponding EM-style update, see Section \ref{sec:impute-em-main}. The formulation makes explicit that the full data distribution is generally non-identifiable from incomplete observations alone.
    \item We instantiate the EM-style update with native mixed-state diffusion backbones for heterogeneous data, combining Gaussian and masked categorical components without one-hot relaxations, see Section \ref{sec:methods-practical-impl}.
    \item We evaluate \ourname\ on mixed-type tabular imputation as the main experiment, with text imputation as a controlled validation of the native discrete backbone in isolation, and study how EM-style refinement affects diffusion-based imputation quality.
\end{itemize}

\paragraph{Notation}
Let \(\mathcal{X}=\mathbb{R}^{d_1}\times\mathcal{S}\), with \(\mathcal{S}=\prod_{\ell=1}^{d_2}\mathcal{S}_\ell\), be a mixed continuous-discrete data space, where \(\mathbb{R}^{d_1}\) contains continuous coordinates and \(\mathcal{S}\) collects the discrete coordinates. Each \(\mathcal{S}_\ell\) is a finite discrete set, such as a categorical or binary variable. We write \(D=d_1+d_2\) for the total number of variables and denote a sample by \(x=(x^{(1)},\ldots,x^{(D)})\in\mathcal{X}\). The set of probability distributions on \(\mathcal{X}\) is denoted by \(\mathcal{P}(\mathcal{X})\). A mask is a binary vector \(m\in\{0,1\}^{D}\), where \(m_j=1\) indicates that coordinate \(j\) is observed and \(m_j=0\) indicates that it is missing. We write \(x_m=\{x^{(j)}:m_j=1\}\) for observed coordinates and \(\overline{x}_m=\{x^{(j)}:m_j=0\}\) for missing coordinates. For a distribution \(p\in\mathcal{P}(\mathcal{X})\), \(p_m\) denotes its marginal distribution on the observed coordinates indexed by \(m\).

% a $D$-dimensional space over $\mathbb{S}$, where $\mathbb{S}=\{0,1,2,\dots\}$, i.e., $\mathcal{X}=\mathbb{S}^D$. Elements $x\in\mathcal{X}$ are $D$-dimensional vectors, written as $x=(x^1,\dots,x^d,\dots,x^D)$. With a slight abuse of notation, we also use $x$ to denote the corresponding random variable with distribution $p(x)\in\mathcal{P}(\mathcal{X})$. We write $p(x)$ for a generic distribution over $\mathcal{X}$ and use text subscripts to specify the distribution when needed, e.g., $p_{\mathrm{model}}(x)$ for the model distribution and $p_{\mathrm{data}}(x)$ for the data distribution. We use $q$ to denote variational distributions. 
% We denote by $\mathcal{S}$ the set of singleton orders over $\{1,\dots,D\}$, i.e., ordered partitions $\pi=(\pi_1,\dots,\pi_D)$ in which each block $\pi_d$ contains exactly one coordinate. We denote by $\mathcal{B}$ the set of block orders, i.e., sequences $\pi=(\pi_1,\dots,\pi_T)$ where each $\pi_t\subseteq\{1,\dots,D\}$ may be empty, the non-empty blocks are pairwise disjoint, and $\bigcup_{t=1}^T \pi_t=\{1,\dots,D\}$.

\vspace{-2mm}
\section{Background}
\vspace{-2mm}

In this section we do define the learning from missing data problem and describe the methodology of diffusion models, which are our algorithms backbone generative model.

\subsection{Problem Setup}\label{sec:problem-setup}

% Consider data distribution $p^*(x) \in \mathcal{P}(\mathcal{X})$ and let $x = (x^{(1)}, \dots, x^{(D)}) \in \mathcal{X}$ denote a clean sample. In our case, we do not observe full clean data, instead, we observe samples covered by \textbf{observation masks}. An observation mask is a binary vector $m \in \{0,1\}^D$, where $m_j=1$ indicates that coordinate $j$ is observed and $m_j=0$ indicates that coordinate $j$ is missing. The mask induces indexed subvectors of $x$ (observed and missing parts):

Let $p^*(x) \in \mathcal{P}(\mathcal{X})$ denote the data distribution and let $x \in \mathcal{X}$ be a clean sample. In our setting, full samples are not observed directly. Instead, observations are partially revealed through \textbf{observation masks}. Such a mask is a binary vector $m \in \{0,1\}^D$, where $m_j = 1$ means that coordinate $j$ is observed and $m_j = 0$ means that it is missing. The mask partitions $x$ into observed and missing subvectors:

\vspace{-4mm}
\begin{equation}
  \label{eq:xm-def}
  \xobs := \{x^{(j)} : m_j=1\},
  \qquad
  \xmis := \{x^{(j)} : m_j=0\}.
  \end{equation}

  \vspace{-2mm}

% Then we restrict the full set of possible masks to
% \[
% \mathcal{M} \subset \{0,1\}^D,
% \]
We model masks as random variables, with

\vspace{-4mm}
\begin{equation}
    \mu(m) \in \mathcal{P}(\mathcal{M}), \quad \mathcal{M} \subseteq \{0, 1\}^D
\end{equation}

\vspace{-2mm}

% which mirrors real-life scenarios, where the observation pattern is stochastic, i.e., values are stochastically getting missing. The observed data $p_m(x_m)$ is constructed by taking the samples from $p^*(x)$ and drop the indices determined by the mask sampled from  $\mu(m)$:

which reflects real-world settings where missingness is stochastic. The observed distribution is the mask-induced marginal:

\vspace{-5mm}
\begin{equation}
p_m(x_m) = \int p^*(x_m, \overline{x}_m)\, d\overline{x}_m, \quad x \sim p^*,\ m \sim \mu.
\end{equation}
\vspace{-3mm}

When clean sample entries are being dropped we inherently lose information about the clean data distribution $p^*(x)$ and do have only the access to its partial marginals $p^*_m(x_m)$, which in general are insufficient to uniquely identify the $p^*$ distribution \cite{yoon2018gain}. The goal of the data imputation problem is to \textbf{learn a conditional imputer} that samples plausible missing values $\overline{x}_m$ given observed values $x_m$, while matching the observed mask-indexed marginals available in the incomplete data. Two empirical setups are commonly considered: \textit{in-sample} imputation, where the model imputes the training data, and \textit{out-of-sample} imputation, where it is applied to unseen data. 
% In this work, we mainly focus on the in-sample setting.
% Although the data without missing entries can appear in the dataset, the amount of it is considered small or insufficient.

% \textbf{recover either $p^*(x)$ or other joint distribution that follows the observed marginals $p^*_m(x_m)$ data} or equivalently to build a generative model to sample from $p^*(x)$ or other joint distribution. 

% \textcolor{red}{tell about empirical version and remove assumptio that we do not have zeros masks. Tell that in practice small numner of clean sample sstill doesn't allow for full recovery.}

\vspace{-2mm}
\subsection{Diffusion models}\label{sec:diff_background}
\vspace{-2mm}

Diffusion models learn a data-generating reverse process by inverting a prescribed forward corruption \citep{ho2020denoising,song2021sde}. The construction depends on the state space: continuous diffusion adds Gaussian noise for real-valued features \citep{ho2020denoising,song2021sde}; masked discrete diffusion operates natively on categorical variables \citep{austin2021structured,sahoo2024mdlm}; mixed-state diffusion combines both for heterogeneous data \citep{tabdiff}. For imputation, this matters because observed coordinates should stay in their native domains while missing coordinates are resampled conditionally. Below, we describe the DDPM \cite{ho2020denoising} and MDM \cite{austin2021structured} formulations adopted in our work.

\paragraph{Training}

For continuous data, such as real-valued tabular features, diffusion models such as DDPM \cite{ho2020denoising} or Score SDE \cite{song2021sde} define a forward process that gradually adds Gaussian noise to a clean sample, from clean data $x_0$ to nearly pure noise $x_T$:

\vspace{-5mm}
\begin{equation}
x_t = \sqrt{\bar\alpha_t} x_0 + \sqrt{1-\bar\alpha_t}\epsilon,
\qquad
\epsilon \sim \mathcal{N}(0,I).
\end{equation}

\vspace{-2mm}

The reverse transitions are learned through the DDPM training objective, which is commonly written as the noise-prediction loss:

\vspace{-5mm}
\begin{equation}
\mathcal{L}_{\rm ddpm}(\theta)
=
\mathbb{E}_{x_0,t,\epsilon}
\left[
\left\|\epsilon-\epsilon_\theta(x_t,t)\right\|^2
\right].
\end{equation}

\vspace{-2mm}

For discrete data, such as text, categorical variables, or tokens, Gaussian noise is not naturally defined. Instead, masked diffusion models corrupt samples through masked transitions:

\vspace{-5mm}
\begin{equation}
\begin{aligned}
q_{\rm MDM}(x_t=a\mid x_0)
={}& \alpha_t \mathbf{1}\{a=x_0\} + (1-\alpha_t)\mathbf{1}\{a=M\},
\end{aligned}
\end{equation}

\vspace{-2mm}

where $a \in \mathcal{S}\cup\{M\}$ and \(M\) is a special mask token. The reverse transitions are trained by maximizing the probability assigned to the original discrete value:

\vspace{-5mm}
\begin{equation}
\mathcal{L}_{\rm mdm}(\theta)
=
\mathbb{E}_{x_0,t,x_t}
\left[
  -\log \hat p_\theta(x_0 \mid x_t)
\right],
x_t \sim q_{\rm MDM}(x_t\mid x_0). \nonumber
\end{equation}
\vspace{-4mm}

This formulation keeps discrete variables finite-valued throughout the diffusion process.

% The learned reverse process then iteratively predicts or resamples masked entries conditioned on the current corrupted state:
% \[
% p_{\theta}(x_{t-1}\mid x_t)
% =
% \mathrm{Cat}\!\left(
% x_{t-1};
% \frac{(1-\alpha_{t-1})M+(\alpha_{t-1}-\alpha_t)p_{\theta}(x_0\mid x_t)}
% {1-\alpha_t}
% \right).
% \]

\paragraph{Generation process}

In the continuous case, the learned reverse process then estimates the denoising direction at each noise level, ultimately transforming pure noise into a realistic data point:

\vspace{-5mm}
\begin{equation}
\begin{aligned}
p_\theta(x_{t-1}\mid x_t)
&= \mathcal{N}\!\left(x_{t-1};\mu_\theta(x_t,t),\sigma_t^2 I\right),\\
\mu_\theta(x_t,t)
&= \frac{1}{\sqrt{\alpha_t}}
\left(x_t-\frac{\beta_t}{\sqrt{1-\bar\alpha_t}}\epsilon_\theta(x_t,t)\right).
\end{aligned}
\end{equation}
\vspace{-2mm}

In the discrete case, the learned reverse process iteratively predicts masked entries conditioned on the current corrupted state:

\vspace{-5mm}
\begin{equation}
\begin{aligned}
&p_{\theta}(x_{t-1}\mid x_t)\\
&\quad =
\mathrm{Cat}\!\left(
x_{t-1};
\frac{(1-\alpha_{t-1})M+(\alpha_{t-1}-\alpha_t)p_{\theta}(x_0\mid x_t)}
{1-\alpha_t}
\right).
\end{aligned} \nonumber
\end{equation}
\vspace{-2mm}

\begin{table*}[t]
  \centering
  \small
  \begin{tabular}{lccl}
  \toprule
  \textbf{Method} & \textbf{Backbone Model} & \textbf{Practical Discrete Modeling} & \textbf{Empirical focus} \\
  \midrule
  GAIN \cite{yoon2018gain}       & GAN & \poscell & Tabular imputation \\
  MIRI \cite{yu2025missing_miri} & Rectified Flow & \negcell & Tabular imputation \\
  MCFlow \cite{richardson2020mcflow} & Normalizing Flow & \midcell & Tabular imputation \\
  Diffputer \cite{diffputer} & Score Diffusion \cite{song2021sde} & \midcell & Tabular imputation \\
  DiffEM \cite{hosseintabar2025diffem} & Score Diffusion \cite{song2021sde} & \negcell & Image reconstruction \\
  Impute-EM (\textbf{ours}) & General Diffusion \cite{ho2020denoising, austin2021structured, tabdiff} & \poscell & Text \& mixed-type tabular \\
  \bottomrule
  \end{tabular}
  \vspace{1mm}
  \caption{Conceptual comparison of iterative, i.e., EM-like, imputation methods. The discrete variables column indicates whether the method handles discrete variables directly, where \ding{51}\ding{55} denotes a one-hot encoding workaround. The empirical focus column reports each method's main evaluation domain.}
  \label{tab:conceptual-comparison}
  \vspace{-6mm}
  \end{table*}

\paragraph{Mixed-state diffusion models}

Many real-world applications contain both continuous and discrete variables. Mixed space diffusion models combine the two constructions above by applying Gaussian diffusion to continuous coordinates and masked categorical diffusion to discrete coordinates. The forward process is obtained by concatenating the corresponding corruptions, and the reverse process jointly denoises or resamples both parts. Training typically uses a summed objective,
\[
\mathcal{L}_{\rm mixed}(\theta)
=
\mathcal{L}_{\rm ddpm}(\theta)
+
\mathcal{L}_{\rm mdm}(\theta),
\]
possibly with task-dependent weights between the continuous and discrete components.

\vspace{-2mm}
\section{Related Work}
\vspace{-2mm}

Data imputation methods aim to estimate plausible values for missing entries, or a conditional distribution over them, when values are lost during acquisition or were never measured. Imputation is central for heterogeneous tabular data \citep{yoon2018gain,du2024remasker}, token-based audio restoration \citep{dror2025token}, graph attributes \citep{you2020handling}, and other data mining settings with mixed variable types.

\paragraph{Classical methods} Early imputation methods include nearest-neighbor imputation \citep{pujianto2019knn}, Gaussian mixtures \citep{garcia2010pattern}, and HyperImpute automatic model selection \citep{hyperimpute}. More recently, deep generative models have framed imputation as conditional generation of missing values from observed entries. VAE-based methods such as MIWAE \citep{mattei2019miwae} and GP-VAE \citep{fortuin2020gpvae} model incomplete observations with latent variables, while GAIN \citep{yoon2018gain} uses adversarial training. Diffusion-based approaches such as TabCSDI \citep{zheng2022tabcsdi} and MissDiff \citep{ouyang2023missdiff} sample missing values conditional on observed entries. Many contemporary generative approaches, including ReMasker \citep{du2024remasker}, MissDiff \citep{ouyang2023missdiff}, and MIWAE \citep{mattei2019miwae}, can be viewed as optimizing an incomplete-data likelihood, where the training objective is evaluated only on the observed components of each sample.

\paragraph{Iterative methods} Iterative refinement is another long standing strategy for imputation, where initial guesses for missing values are repeatedly improved. The EM algorithm is the classical example \citep{dempster1977em}, although early uses often relied on simple distributions such as Gaussian mixtures, Bernoulli models, or multinomial models \citep{garcia2010pattern}. Related modern methods include MCFlow, which iteratively imputes with normalizing flows \citep{richardson2020mcflow}, IGRM, which updates graph based friend networks during training \citep{zhong2023igrm}, and HyperImpute, which iteratively refines model selection and imputations \citep{hyperimpute}. MIRI also uses an EM based iterative algorithm, but views imputation as reducing mutual information between the completed data and the missingness mask with rectified flows \citep{yu2025missing_miri}.

DiffPuter trains an EM-style iterative improvement procedure with a diffusion model for missing-data imputation \citep{diffputer}. DiffEM learns continuous diffusion models from corrupted observations by alternating conditional reconstruction in the E step and score matching in the M step \citep{hosseintabar2025diffem}. DiffEM’s EM analysis covers general state spaces and guarantees observation consistency without identifiability, while its diffusion models and experiments remain continuous. Our work is closest to DiffPuter and DiffEM: these methods show that iterative reconstruction and model refitting can improve diffusion-based imputation or reconstruction, but they primarily use continuous diffusion backbones. Our focus is different: we instantiate the same EM-style principle with native mixed-state diffusion backbones for heterogeneous data.

AugMask \citep{kim2026augmask} also studies training tabular diffusion models, including TabDiff, from incomplete data. Unlike our iterative EM approach, AugMask constructs fixed stochastic completions with auxiliary feature-wise models and uses them only as conditioning context under an observed-only denoising loss.

\paragraph{Disclaimer}
This work was completed and submitted to ICDM 2026 by June 6, 2026, but could not be posted on arXiv during the review process under the conference policy.

\input{chapters/methods_em.tex}

\input{chapters/experiments.tex}

\vspace{-1mm}
\section{Conclusion, Limitations, and Future Work}\label{sec:potential-impact-future-work}
\vspace{-1mm}

We introduced \ourname, an EM-style framework that instantiates diffusion imputation with native backbones for heterogeneous data, and showed empirically that it improves imputation quality on discrete text and achieves the best distributional metrics and the best average downstream rank on heterogeneous tabular data.

Our method has two main limitations, both shared with other diffusion-based EM imputers \cite{yoon2018gain,diffputer,hosseintabar2025diffem}: missing-data learning is non-identifiable, so we guarantee observed-marginal consistency rather than recovery of the true distribution, and each EM iteration requires conditional diffusion sampling followed by training, which is more expensive than single-stage imputation.

Future work includes faster conditional samplers, regularization to select among compatible distributions, and adaptation to other heterogeneous domains.

\bibliographystyle{IEEEtranN}
\bibliography{references}

%%%%%%%%%%%%%%%%%%%%%%%%%%%%%%%%%%%%%%%%%%%%%%%%%%%%%%%%%%%%

\input{chapters/appendix.tex}

%%%%%%%%%%%%%%%%%%%%%%%%%%%%%%%%%%%%%%%%%%%%%%%%%%%%%%%%%%%%

% NeurIPS paper checklist is not used for ICDM.
% \newpage
% \input{chapters/checklist.tex}

\end{document}

%% file: chapters/methods_em.tex
\vspace{-2mm}
\section{Method}
\vspace{-2mm}

This section presents \ourname. We first formulate missing-data learning as matching the observed marginals induced by a family of masks Section \ref{sec:impute-em-main}. We then describe the practical instantiation with native state space backbones Section \ref{sec:methods-practical-impl}. Finally, we characterize the exact nonparametric update and show convergence of the observed marginals as a supporting theoretical result Section \ref{sec:impute-em-theory}.

\subsection{Learning from missing data with \ourname} \label{sec:impute-em-main}
% 
% \paragraph{Objective}

Given the data distribution $p^*(x) \in \mathcal{P}(\mathcal{X})$ with the observation masks probability distribution $ \mu(m) \in \mathcal{P(\mathcal{M})}$, where masks are $m \in \{0,1\}^D$. 
% We consider the case where the masks $\mathcal{M}$ do not cover the whole space of masks $2^{[D]}$ and furthermore do not include the non mask case where data samples stay clean. In other words, we \textbf{do not see the clean data samples} $x \sim p^*(x)$. 
Following problem setup Section \ref{sec:problem-setup}, we observe only partial random vectors $x_m \sim p_m^*(x_m)$. The full distribution $p^*(x)$ is generally not identifiable from incomplete observations alone \cite{yoon2018gain}. Therefore, our objective is to learn a model $p(x)$ whose observed marginals agree with the target marginals $p_m^*(x_m)$ for the mask family used during training.

We treat $\mu(m)$ as independent of $x$ (missing completely at random, MCAR), which ensures the observed marginals $p_m^*(x_m)$ are well-defined. In practice, \ourname\ can also be applied under MAR or MNAR: the E-step and M-step remain computationally valid, though under MNAR the observed marginals are biased by the missingness mechanism and the theoretical guarantee does not apply.

This leads to the following optimization problem:

% fit the generative model $p^\theta$ to follow the observed from data distribution $p^*(x_m)$ samples as close as possible. In that light we define the following objective:

\vspace{-4mm}
\begin{equation}\label{eq:main-objective}
\begin{aligned}
&\argmin_{p} \mathbb{E}_{m \sim \mu} \left[ \KL{p^*_m(x_m)}{p_m(x_m)} \right] \\
&\qquad = \argmax_{p} \mathbb{E}_{\substack{m \sim \mu\\ x_m\sim p^*_m(x_m)}}\!\left[\log p_m(x_m)\right] + C_1.
\end{aligned}
\end{equation}
\vspace{-2mm}

% \paragraph{Variational Objective}

Where $C_1$ is an entropy term that does not depend on $p$. To solve this problem we can follow the \textbf{variational inference} approach and introduce a variational distribution $q(\overline{x}_{m}|x_{m})$ over the missing entries $\overline{x}_{m}$, which are treated as latent variables:

\vspace{-6mm}
\begin{equation}
\log p(x_m) \geq \mathbb{E}_{\overline{x}_m \sim q(\overline{x}_{m}|x_{m})} \left[\log \frac{p({\overline{x}_m}, x_{m})}{q({\overline{x}_{m}|x_{m})}}\right] = \mathcal{L}(p, q|x_m),
\end{equation}
\vspace{-4mm}

where the $\argmax_{q} \mathcal{L}(p, q)$ solution would be $q(\overline{x}_{m}|x_{m}) = p(\overline{x}_{m}|x_{m})$, for each $m$, which gives us the Expectation step. Then the full objective with lower bound:

\vspace{-4mm}
\begin{equation}
\begin{aligned}
    \mathbb{E}_{m \sim \mu,\; x_m \sim p_m^*}
    \left[\log p(x_m)\right]
    &\geq
    \mathbb{E}_{m,\; x_m \sim p_m^*}
    \left[\mathcal{L}(p, q \mid x_m)\right] \\
    &=
    \mathbb{E}_{\substack{m,\; x_m \sim p_m^*\\ \overline{x}_m \sim q(\overline{x}_m \mid x_m)}}
    \left[
        \log
            p(\overline{x}_m, x_m)\right]
         + C_2.
\end{aligned}
\label{eq:variational-lower-bound}
\end{equation}
\vspace{-4mm}

where $C_2$ is an entropy term that does not depend on $p$. 

With fixed $q$, the rhs of \eqref{eq:variational-lower-bound} is just the log likelihood for the $p(\overline{x}_{m}, x_{m})$ w.r.t. data distribution $x = \{\overline{x}_{m}, x_{m}\} \sim q(\overline{x}_{m}|x_{m})p_m^*(x_m)$.  Then we can find the solution for $\argmax_{p} \mathcal{L}(p, q)$ by the regular likelihood optimization for the model $p(x)$, which gives us the Maximization step. We call the resulting EM algorithm \textbf{\ourname}:

% $q(\overline{x}_{m}|x_{m}) = p(\overline{x}_{m}|x_{m})$

% \begin{itemize}
% \item \textbf{E}xpectation step: $q^n(\overline{x}_{m}|x_{m}) \leftarrow p^{n}_m(\overline{x}_{m}|x_{m})$
% \item \textbf{M}aximization step: $p^{n+1} \leftarrow \argmax_{p} \mathbb{E}_{m \sim \mu,\;x_m\sim p^*} \left[\mathbb{E}_{\overline{x}_{m} \sim q^n (\overline{x}_{m}|x_{m})} \log p(\overline{x}_{m}, x_{m})\right]$
% \end{itemize}

\begin{itemize}
\item \textbf{E}xpectation step: $q^n(\overline{x}_{m}|x_{m}) \leftarrow p^{n}_m(\overline{x}_{m}|x_{m})$.
\item \textbf{M}aximization step:
\[
\begin{aligned}
p^{n+1} \leftarrow{}& \argmax_{p} \mathbb{E}_{\substack{m \sim \mu,\;x_m\sim p_m^*\\ \overline{x}_{m} \sim q^n (\overline{x}_{m}|x_{m})}} \left[\log p(\overline{x}_{m}, x_{m})\right]\\
={}& \argmax_{p} \mathcal{L}(p, q).
\end{aligned}
\]
\end{itemize}

% Maximization step: $\theta^{n+1} \leftarrow \argmax_{\theta} \mathbb{E}_{m,\;x_m\sim p_m^*} \left[\mathbb{E}_{\overline{x}_{m} \sim q^n_m (\overline{x}_{m}|x_{m})} \log p^{\theta}(\overline{x}_{m}, x_{m})\right]$

\subsection{Native Diffusion Backbones for \ourname} \label{sec:methods-practical-impl}

We instantiate the EM-style update with diffusion backbones that operate in the native state space of each variable, spanning continuous, discrete, and mixed types. The E-step samples missing entries from the current backbone and the M-step retrains it on the completed data.

% One can notice that when choosing the generative model to implement the Iterative Refinement procedure the only requirements are: 1) the ability to sample from the model and 2) the ability to learn the model via the log-likelihood maximization. 

\paragraph{Expectation step} The conditional inference of the current model $p_\theta$ has to be performed at E-step. By our construction we have a conditional diffusion $p_\theta(\overline{x}_{m}|x_{m})$, which in continuous case would be the conditional mean predictor $\mu_{\theta}(\overline{x}_{m, t}, t| x_m)$ or in discrete case the conditional data predictor $\hat p_{\theta}(\overline{x}_m \mid \overline{x}_{m, t}, x_{m})$. In both cases the conditional generation would be just backward diffusion process inference with conditional noise predictor and conditional data predictor, similarly to methodology described in Section \ref{sec:diff_background}. In detail, for the continuous case:

\vspace{-4mm}
\begin{equation*}
  \begin{aligned}
  &p_{\theta^n}(\overline{x}_{m, t-1} \mid \overline{x}_{m, t}, x_m)\\
  &\qquad = \mathcal{N}\!\left(\overline{x}_{m, t-1}; \mu_{\theta^n}(\overline{x}_{m, t}, t \mid x_m), \sigma_t^2 I\right).
  \end{aligned}
\end{equation*}
\vspace{-2mm}

% where $ \mu_{\theta^n}(x_{m, t}, t| x_m) = \frac{1}{\sqrt{\alpha_t}}(x_{m, t} - \frac{\beta_t}{\sqrt{1-\bar\alpha_t}} \epsilon_{\theta^n}(x_{m, t}, t| x_m))$ and $\sigma_t$ is noise coefficient. 

For the discrete case:

\vspace{-4mm}
\begin{equation*}
  \begin{aligned}
  &p_{\theta^n}(\overline{x}_{m, t-1} \mid \overline{x}_{m, t}, x_m) = \\
  &\quad \hspace{-5mm}{\rm Cat}\!\left(\overline{x}_{m, t-1}\,\middle|\, \frac{(1- \alpha_{t-1})M  +  (\alpha_{t-1} - \alpha_t)\, \hat p_{\theta^n}(\overline{x}_m \mid \overline{x}_{m, t}, x_m)}{1-\alpha_t}\right).
  \end{aligned}
\end{equation*}
\vspace{-2mm}

where $\alpha_t$ coefficients which are determined by the forward process construction. Then such the probability distribution $\mathcal{D}^n(x)$, would be the result of E-step:

\vspace{-4mm}
\begin{equation}\label{eq:m-step-data-Dn}
  \mathcal{D}^n(x) = \mathbb{E}_{m \sim \mu}[\mathcal{D}^n(\overline{x}_m, x_m)] = \mathbb{E}_{m \sim \mu}[p_{\theta^n}(\overline{x}_m| x_m)\, p_m^*(x_m)]
\end{equation}
\vspace{-4mm}

In practice, the \textit{finite dataset} is used instead of $\mathcal{D}^n(x)$. Notice, that E-step requires sampling from the learned model, which introduces additional computational overhead for diffusion models \cite{ho2020denoising, austin2021structured}. Similar overhead is inherent to EM-based methods with diffusion backbones \cite{diffputer, hosseintabar2025diffem, yu2025missing_miri}.

\paragraph{Maximization step} During the M step we have to fit the family of conditional generative models $p_\theta(\overline{x}_{m}|x_{m})$ to $\mathcal{D}^n(x)$ data via likelihood maximization or equivalent optimize the following loss function:

% In practice, to add a clean data conditional inference mechanism to the diffusion model, at the M step we learn not the single generative model but the family of consistent generative models conditioned on observed data, i.e.,  $p_\theta(\overline{x}_{m}|x_{m})$. The conditional 
% x_mis, x_obs, m as input to loss

\vspace{-5mm}
\begin{equation}
  \label{eq:parametric-update}
  \begin{aligned}
  \mathcal{L}_{\rm cond}(\theta) &= -\mathbb{E}_{x \sim \mathcal{D}^n(x), m \sim \mu} \left[ \log p_{\theta}(\overline{x}_{m}|x_{m}) \right]
  \end{aligned}
\end{equation}
\vspace{-4mm}

In practice, learning of conditional diffusion models is long standing practice \cite{sahoo2024mdlm} and is done by simply adding another input to the neural network and slightly modifying the loss function. Final loss function for continuous state space

\vspace{-5mm}
\begin{equation}
  \label{eq:diffusion-loss-cont}
  \mathcal{L}_\mathrm{cond}^{\rm ddpm}(\theta) = \mathbb{E}_{\substack{x \sim \mathcal{D}^n(x),\, m \sim \mu\\ t \sim \mathcal{U}[0,T],\, \epsilon \sim \mathcal{N}(0, I)}} \left[
    \left\| \epsilon - \epsilon_\theta(\overline{x}_{m, t}, t| x_{m}) \right\|^2
  \right]
\end{equation}
\vspace{-4mm}

where $\overline{x}_{m, t} = \sqrt{\bar\alpha_t} \overline{x}_m + \sqrt{1-\bar\alpha_t} \epsilon$ represents the forward noising process, and $\epsilon_\theta$ is the model prediction. The conditional structure is enforced by providing the mask $m$ and observed values $x_m$ as inputs to the network.

For discrete state space, we notice that MDM is naturally suited for the imputation problem since it already learns a family of conditional generative models by design \cite{ou2025absorbing} and carries the information about the masked values by special masked token. In that light we can use the neural network architecture as and just slightly alter the  training procedure:

\begin{algorithm}[t]
  \caption{\ourname}
  \label{alg:ambientmdm}
  \KwIn{Initial parameters $\theta^0$, observed data $\{x_m^i\}_{i=1}^L$, number of EM iterations $N$}
  \KwOut{Refined parameters $\theta^N$}
  Initialize $n \leftarrow 0$\;
  \While{$n < N$}{
    Initialize dataset $\mathcal{D}^{n} \leftarrow \emptyset$\;
    \For{each $x_m^i \in \{x_m^i\}_{i=1}^L$}{
      Sample a batch $x_i^{n} = x^i \sim p_{\theta^n}(\overline{x}_m\mid x_m^i)\,p_m^*(x_m^i)$\; \tcp{Inference diffusion $\theta^n$}
      Update dataset $\mathcal{D}^{n} \leftarrow \mathcal{D}^{n} \cup x_i^{n}$\;
    }
    $\theta \leftarrow \theta^n$\;
    \While{\textnormal{not converged}}{
      Update $\theta$ with an optimization step on $\mathcal{L}_{\rm cond}(\theta) = -\mathbb{E}_{x \sim \mathcal{D}^{n},\, k \sim \mu} \left[ \log p_{\theta}(\overline{x}_{k}|x_k, k) \right]$\;
    }
    $\theta^{n+1} \leftarrow \theta$\;
    $n \leftarrow n+1$\;
  }
  \Return $\theta^N$\;
  \end{algorithm}

\vspace{-4mm}
\begin{equation}
  \label{eq:diffusion-loss-disc}
  \mathcal{L}^\mathrm{mdm}_{\rm cond}(\theta) = \mathbb{E}_{\substack{x \sim \mathcal{D}^n(x),\, m \sim \mu\\ t \sim \mathcal{U}[0,T]\\ \overline{x}_{m, t} \sim q_{\rm{MDM}, t}(\overline{x}_{m, t}|\overline{x}_{m})}} \left[\mathrm{CE}\bigl(e_{\overline{x}_{m}}, \hat p_\theta(\cdot| x_{m})\bigr)
  \right]
\end{equation}
\vspace{-5mm}

\paragraph{Heterogeneous mixed-state diffusion}
In our implementation, the mixed state space model is built by concatenating continuous and discrete diffusion components, and the neural networks take the full multimodal input in one forward pass. The joint training objective would be just the sum of continuous and discrete losses:

\vspace{-4mm}
\begin{equation}
\mathcal{L}_{\rm cond}(\theta)
=
\mathcal{L}_{\rm cond}^{\rm ddpm}(\theta)
+
\mathcal{L}_{\rm cond}^{\rm mdm}(\theta),
\end{equation}
\vspace{-2mm}

The other details including the full joint training objective and reverse sampling updates are delivered in Appendix~\ref{app:add-methods}.

\paragraph{Initial model} Initialization is an important practical choice for \ourname. We set $p_0^\theta$ to diffusion models trained via the incomplete likelihood objectives, see Appendix~\ref{app:add-methods} for more details. On continuous domain we use training close to MissDiff \cite{ouyang2023missdiff}. On discrete domain we train a masked diffusion model as an incomplete data masked autoencoder, analogous to ReMasker \cite{du2024remasker}.

We summarize the resulting practical procedure in Algorithm~\ref{alg:ambientmdm}, which alternates E-step conditional diffusion sampling with M-step conditional training.

\subsection{Observed-marginal consistency of the exact update}\label{sec:impute-em-theory}

We now characterize the exact nonparametric operator underlying Algorithm~\ref{alg:ambientmdm} and state the observed-marginal consistency result that supports it. The analysis is idealized: it assumes exact E-step sampling and exact M-step maximization. The practical diffusion approximations introduced in Section \ref{sec:methods-practical-impl} relax both.

% \paragraph{Exact EM operator at the distribution level.}

\begin{proposition}[\ourname\ update]\label{prop:impute-em-update}

Let $\mathcal M \subseteq \{0, 1\}^D$ be the family of observed index sets, and $ \mu(m) \in \mathcal{P(\mathcal{M})}$ be the positive probability distribution of observed indices. For a joint distribution $p \in \mathcal P(\mathcal X)$, define the exact \ourname\ update.

\vspace{-3mm}
\begin{equation}
p^{n+1}(x) = \mathbb{E}_{m \sim \mu} \left[ \, p_m^*(x_m) p^{n}(\overline{x}_m\mid x_m) \right],
\label{eq:em-operator-general}
\end{equation}
\vspace{-5mm}

where $p_m^*$ denotes the target marginal on the observed coordinates indexed by $m$.

\end{proposition}

See the proof in Appendix~\ref{app:proofs}. This proposition gives the closed-form update operator. The following theorem shows that at the limit of iterations the observed marginals converge to the targets.

\begin{figure}[t]
  \centering
  \begin{subfigure}[t]{\linewidth}
    \centering
    \includegraphics[width=0.75\linewidth]{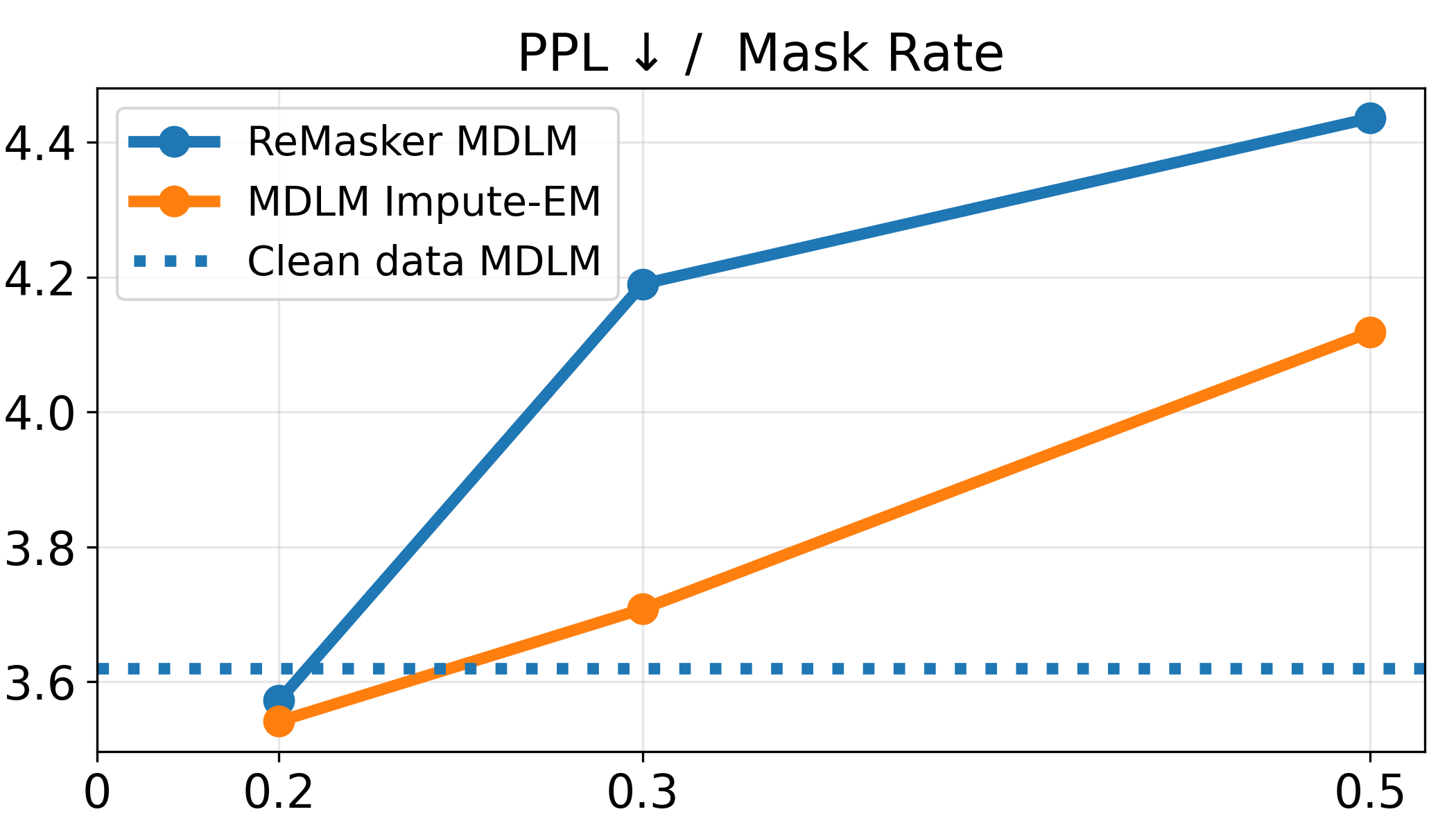}%
    \caption{PPL $\downarrow$ versus mask rate. ReMasker MDLM and Clean data MDLM as baselines, compared with MDLM \ourname\ (ours).}
    \label{fig:text8-ppl-mask}
  \end{subfigure}\\[1ex]
  \begin{subfigure}[t]{\linewidth}
    \centering
    \includegraphics[width=0.75\linewidth]{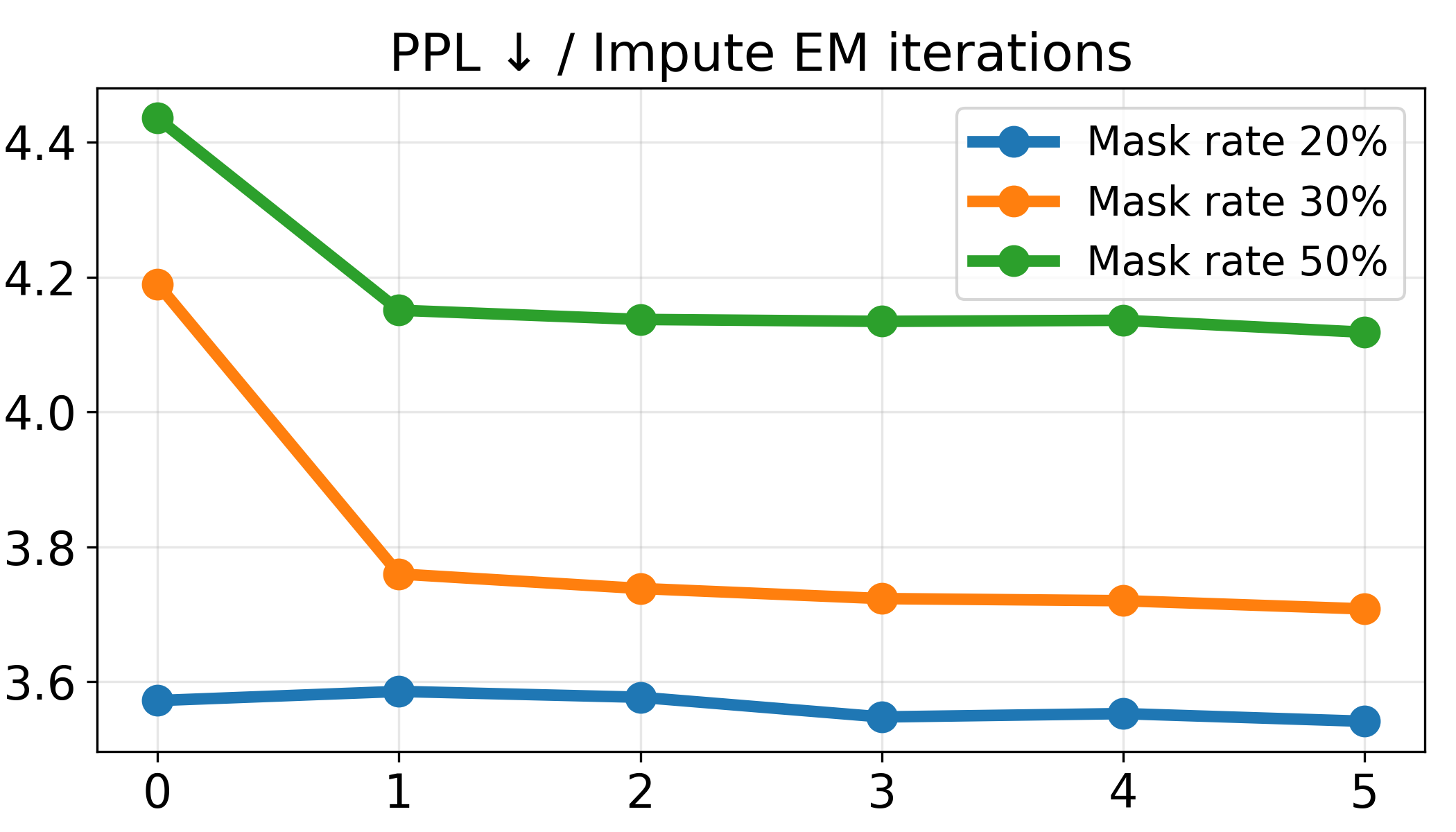}%
    \caption{PPL $\downarrow$ versus number of \ourname\ iterations across mask rates.}
    \label{fig:text8-ppl-em}
  \end{subfigure}
  \caption{Text data imputation on text8 under MCAR per token.}
  \label{fig:text8-ppl}
\vspace{-4mm}
\end{figure}

  \begin{theorem}[Observed-marginal consistency of exact \ourname]\label{thm:lyapunov-inequality}
    Let $p^*(x) \in \mathcal{P}(\mathcal{X})$ be the data probability distribution, $\mathcal M \subseteq \{0, 1\}^D$ be the family of observed index sets and 
    $ \mu(m) \in \mathcal{P(\mathcal{M})}$ be the positive probability distribution of observed indices. 
    %Recall that we denote by $p_m$ corresponding marginal of the distribution $p$.
    Define the set of \eqref{eq:main-objective} minimizers:

    \vspace{-3mm}
    \begin{equation}
    \mathcal C
    :=
    \{q \in \mathcal P(\mathcal X) : q_m = p_m^* \;\; \forall m \in \mathcal M\}.
    \end{equation}
    \vspace{-3mm}

    Define starting probability distribution $p^0(x)$, such that there exists $r\in\mathcal C$ with $\KL{r}{p^0}<\infty$, and define iterative updates:

    \vspace{-3mm}
    \begin{equation*}
    p^{n+1}(x) = \mathbb{E}_{m \sim \mu} \left[ \, p_m^*(x_m) p^{n}(\overline{x}_m\mid x_m) \right].
    \end{equation*}
    \vspace{-3mm}

    Then, for every $m \in \mathcal M$,

    \vspace{-3mm}
    \begin{equation}
    \KL{p_m^*}{p^n_m} \to 0
    \qquad\text{as } n \to \infty.
    \end{equation}
    \vspace{-6mm}

    \end{theorem}

The proof is presented in Appendix~\ref{app:proofs}. The Theorem \ref{thm:lyapunov-inequality} provides that if the \ourname\ procedure is ran long enough, the $p^n$  is guaranteed to \textit{that satisfy the observed marginals $p_m^*$ for all $m \in \mathcal M$}. This property is independent of the start of iteration $p_0$ and the mask probability distribution $\mu(m)$, under slight assumptions. However, there could be many such possible distributions $p$ with fitting marginals, i.e., $p \in \mathcal{C}$, and our theoretical result does not provide any information on what particular distribution would be the result of \ourname\ iteration, i.e., it is not necessarily $p^*$. The theorem concerns the exact nonparametric operator. Its diffusion implementation is a parametric surrogate, so the result supports the target update but does not guarantee convergence of the finite-sample implementation.
% and 2) any information on the dependence of result of \ourname\ iteration and starting distribution $p_0$.

% In addition, for the modern generative models such an algorithm was studied only in the context of continuous state space and research on discrete state space and furthermore mixed state space was not covered extensively.

% \paragraph{Computational complexity.} In practice, the E-step requires sampling from the learned model, which introduces additional computational overhead for diffusion models \cite{ho2020denoising, austin2021structured}. Similar overhead is inherent to EM-based methods with diffusion backbones \cite{diffputer, hosseintabar2025diffem, yu2025missing_miri}.

%% file: chapters/experiments.tex
\vspace{-2mm}
\section{Experiments}\label{sec:exp-main}
\vspace{-2mm}

\begin{figure}[t]
  \centering
  \includegraphics[width=\linewidth]{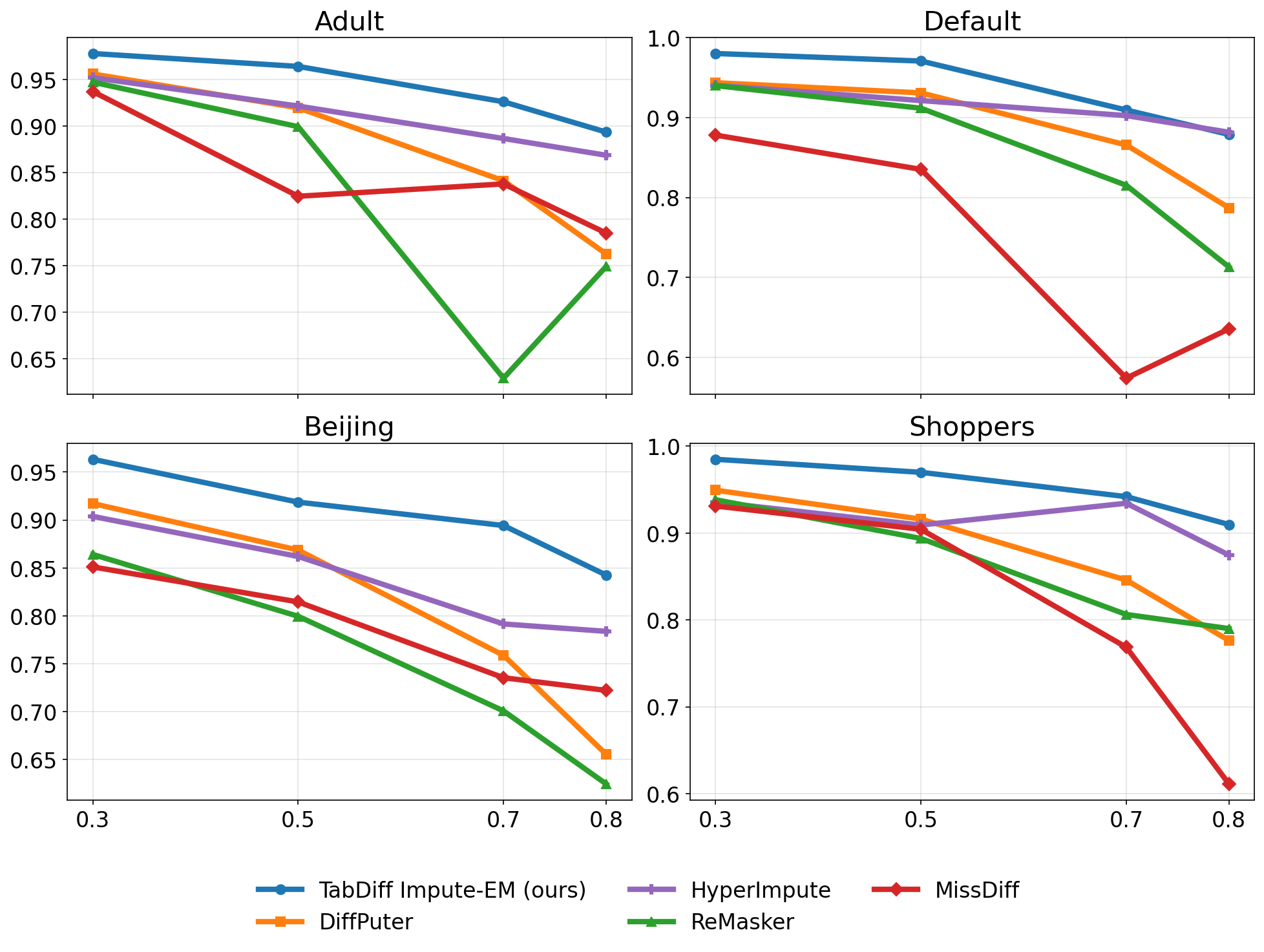}%
  \caption{Overall Density $\uparrow$ under MCAR versus missing rate (Adult, Shoppers, Default, and Beijing). Overall Density is $(\mathrm{Trend}+\mathrm{Shape})/2$.}
  \label{fig:tabular-density-overall}
\vspace{-4mm}
\end{figure}

\begin{figure}[t]
  \centering
  \includegraphics[width=\linewidth]{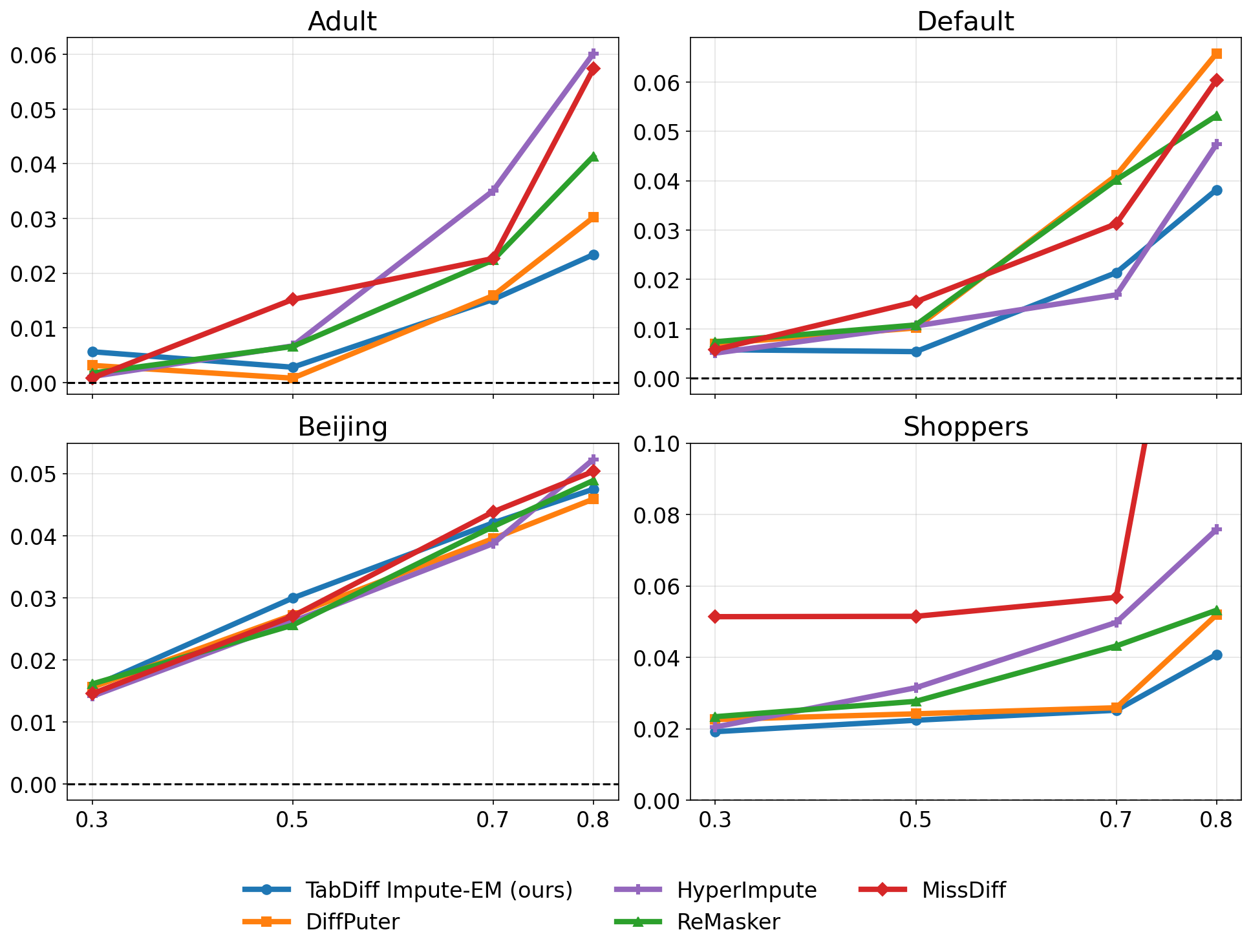}%
  \caption{ML-efficiency absolute deviation $\downarrow$ under MCAR versus missing rate (Adult, Shoppers, Default, and Beijing), relative to the MLE on the clean dataset. Lower is better. Beijing MLE is multiplied by $0.1$ to keep the scale comparable across datasets.}
  \label{fig:tabular-mle}
\vspace{-4mm}
\end{figure}

In this section, we evaluate \ourname\ on heterogeneous, mixed-type data imputation. Our central claim concerns this mixed-type setting, so the main experiment is on tabular data with TabDiff \cite{tabdiff}, which exercises the full mixed continuous-categorical backbone. Since, to our knowledge, native discrete diffusion has not previously been used as an EM-style train-from-incomplete-data backbone for imputation, we additionally study text imputation with a Masked Diffusion Language Model \cite{sahoo2024mdlm} on text8 as a controlled validation of this backbone in isolation, the degenerate case where all coordinates are categorical ($d_1 = 0$). Our problem statement is generative, the goal is to recover the imputation probability distribution $p(\overline{x}_m|x_m)$ rather than individual entries, so we evaluate with generative metrics (PPL for text, Trend and Shape for tabular data), and additionally report downstream ML-efficiency for practical utility.

\subsection{Discrete data: text}

We use the `text8` dataset \cite{mahoney2006text8}, tokenized at the character level, with vocabulary size $27$ and sequence length $128$. We consider several missing rates $p \in \{0.2, 0.3, 0.5\}$ and randomly replace tokens in the dataset samples with a special mask token with probability $p$ (MCAR per token). The missing tokens are generated only once for each data sample. Training is done on the train split and the imputation is done on validation split, following out-of-sample imputation setup.

We compare our method with the ReMasker MDLM, an MDLM trained as a masked autoencoder on observed data only via the incomplete likelihood following the ReMasker \cite{du2024remasker} methodology. The ReMasker MDLM and the initialization of \ourname\ were trained for $500k$ iterations, while MDLM \ourname\ (ours) ran 5 EM iterations of $20k$ training iterations each. All the experimental details can be seen in Appendix~\ref{app:exp-details}.

Our metric is perplexity (PPL, lower is better), the exponential of the average negative log-likelihood per token. PPL scores how well the model's per-token predictive distribution matches held-out tokens, which makes it both the standard distributional metric for masked language modeling \cite{sahoo2024mdlm} and directly tied to token-wise accuracy, hence suitable for measuring imputation ability.
See Appendix~\ref{app:exp-details} for the formal definition. The results are shown in Figure \ref{fig:text8-ppl}. It is evident that missing rates more than $p=0.2$ yield a degradation in model performance which results in significant increase in perplexity. However, MDLM \ourname\ (ours) allows to \textbf{recover the losses in perplexity} w.r.t. Clean MDLM model, which is clearly visible from the Figure \ref{fig:text8-ppl-mask}. Additionally, we study the behavior of \ourname\ by the iterations number and report the results in Figure \ref{fig:text8-ppl-em}. It is clearly visible that \ourname\ iterations improve the imputation quality and the first few \ourname\ steps yield the most significant improvements.

\subsection{Heterogeneous data: tabular}\label{sec:experiments-tabular}

Having validated the discrete backbone in isolation, we now turn to the full heterogeneous setting. We combine the naturally mixed-state TabDiff backbone~\cite{tabdiff} with \ourname, avoiding the one-hot relaxations used by some prior work~\cite{zheng2022tabcsdi, ouyang2023missdiff}.

We consider four mixed-type datasets with both numerical and categorical features: Adult, Shoppers, Default, and Beijing (see Appendix \ref{app:exp-details} for descriptions). As baselines, we compare state-of-the-art approaches: DiffPuter \cite{diffputer} as a diffusion-based EM method, ReMasker \cite{du2024remasker} as a masked autoencoder baseline, and MissDiff \cite{ouyang2023missdiff} as an incomplete likelihood diffusion approach. Both training and imputation take place on the training data (in-sample imputation), and we run Tabdiff \ourname\ for 5 iterations. Further implementation details are in Appendix~\ref{app:exp-details}.

For the wide performance ablation we follow the standard tabular imputation benchmark under the Missing Completely At Random (MCAR) setting, varying the missing rate across $p \in \{0.2, 0.3, 0.5, 0.7\}$. We report \emph{Overall Density} (the average of Shape and Trend marginal and dependency scores~\cite{tabdiff}) for generative fidelity, and ML-efficiency for downstream utility. Both are described in Appendix~\ref{app:exp-details}.

To assess robustness under more challenging missingness, we additionally evaluate all methods under the Missing Not At Random (MNAR) setting on the same four datasets at a single high mask rate $p=0.7$. The MNAR mechanism is described in Appendix~\ref{app:exp-details}. All models are trained with the same hyperparameters as in the MCAR experiments. Figure~\ref{fig:tabular-mnar} reports Overall Density and ML-efficiency deviation, and the discriminative-metric counterpart is in Appendix~\ref{app:add-exps}. TabDiff \ourname\ (ours) still holds its positions in this setting. Additional single-run results under MAR are reported in Appendix~\ref{app:mar-results}.

\begin{figure}[t]
  \centering
  \includegraphics[width=0.9\linewidth]{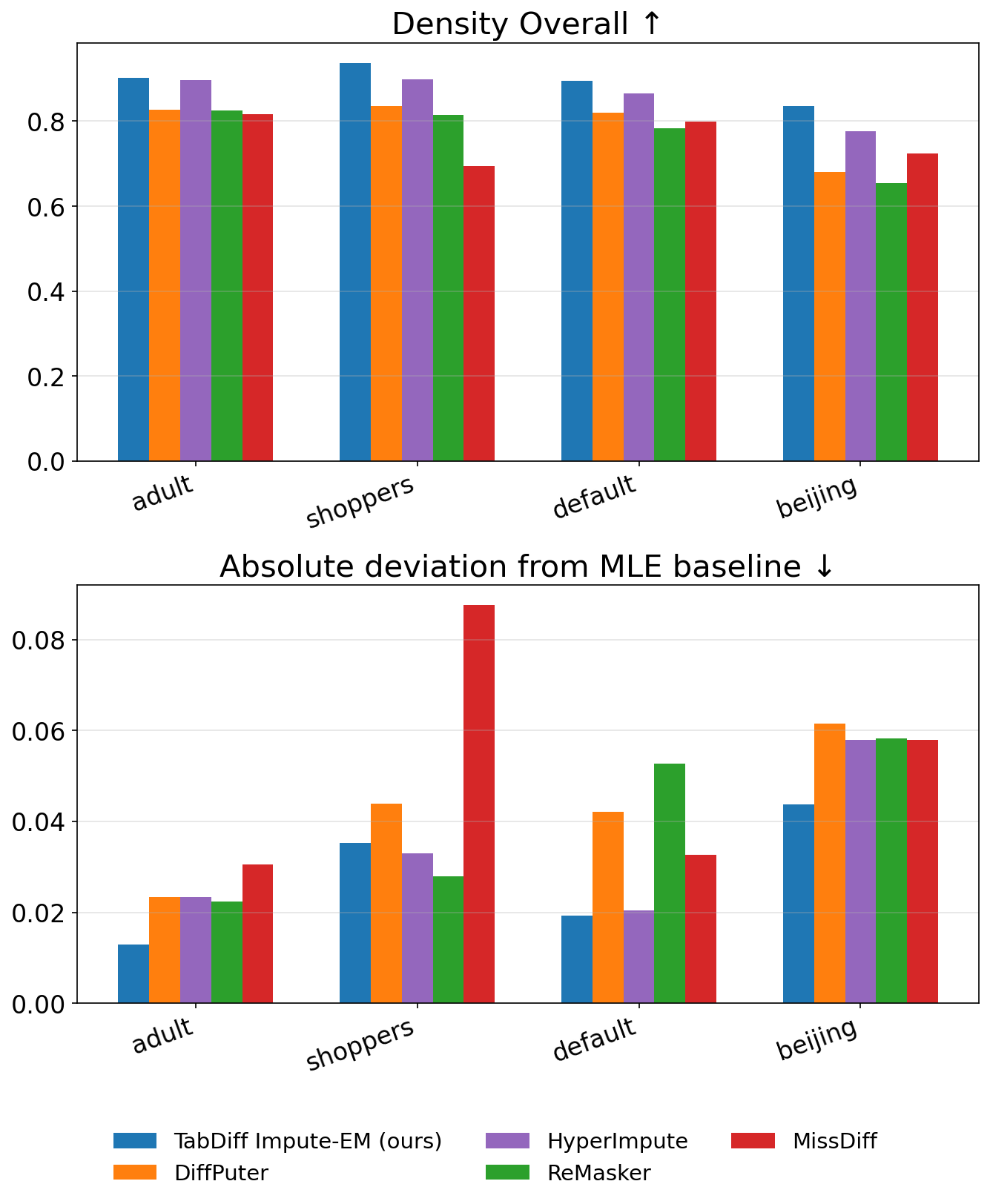}%
  \caption{Under MNAR at mask rate $p=0.7$: Overall Density $\uparrow$ and ML-efficiency absolute deviation $\downarrow$ across the four datasets. Beijing MLE is multiplied by $0.1$ to keep the scale comparable across datasets.}
  \label{fig:tabular-mnar}
\vspace{-2mm}
\end{figure}

\begin{table}[h]
  \centering
  \footnotesize
  \setlength{\tabcolsep}{4pt}
  \begin{tabular}{lccccc}
    \toprule
    Method & Adult & Shoppers & Default & Beijing & Avg. \\
    \midrule
    ReMasker & 0.194 & 0.143 & 0.155 & 0.253 & 0.186 \\
    DiffPuter & 0.136 & 0.133 & 0.118 & 0.207 & 0.148 \\
    HyperImpute & 0.093 & 0.086 & 0.089 & 0.165 & 0.108 \\
    MissDiff & 0.154 & 0.196 & 0.269 & 0.219 & 0.210 \\
    TabDiff \ourname\ (ours) & \textbf{0.060} & \textbf{0.050} & \textbf{0.071} & \textbf{0.102} & \textbf{0.071} \\
    \bottomrule
  \end{tabular}
  \vspace{1mm}
  \caption{Overall Density error $\downarrow$ on tabular datasets under MCAR. For each dataset the error is averaged over the mask rates. The last column reports the averaged error. The best method is highlighted in \textbf{bold}.}
  \label{tab:density-overall-error}
\vspace{-4mm}
\end{table}

\begin{table}[h]
  \centering
  \footnotesize
  \setlength{\tabcolsep}{4pt}
  \begin{tabular}{lccccc}
    \toprule
    Method & Adult & Shoppers & Default & Beijing & Rank \\
    \midrule
    ReMasker & 3 & 3 & 2 & 2 & \circledavg{3} \\
    DiffPuter & 2 & 2 & 4 & \textbf{1} & \circledavg{2} \\
    HyperImpute & 5 & 4 & 2 & 2 & \circledavg{4} \\
    MissDiff & 4 & 4 & 3 & 4 & \circledavg{4} \\
    TabDiff \ourname\ (ours) & \textbf{1} & \textbf{1} & \textbf{1} & 3 & \circledavg{\textbf{1}} \\
    \bottomrule
  \end{tabular}
  \vspace{1mm}
  \caption{Ranking downstream performance on tabular datasets under MCAR. Ranks follow the averaged absolute deviation from the clean-data ML-efficiency in each dataset. The last column ranks methods by their average rank across datasets. Lower is better. The best method is highlighted in \textbf{bold}.}
  \label{tab:mle-ranking}
\vspace{-4mm}
\end{table}
% The distributions metric results are shown in Figure \ref{fig:tabular-density-overall} and Table~\ref{tab:density-overall-error}. While downstream performance is shown in Figure \ref{fig:tabular-mle} and Table~\ref{tab:mle-ranking}. 

% To isolate the contribution of \ourname\ from the TabDiff backbone, Appendix~\ref{app:add-exps} reports an iteration-dynamics analysis across generative, downstream, and discriminative metrics, showing that \ourname\ updates substantially improve $3$ of $4$ of them over the plain TabDiff (incomplete-likelihood) baseline. The same appendix contains runtime comparisons and additional metrics.

To isolate the contribution of \ourname\ from the TabDiff backbone, Appendix~\ref{app:add-exps} reports an iteration-dynamics analysis across different metrics, showing that \ourname\ iterative updates substantially improve majority of them over the plain TabDiff (incomplete-likelihood) baseline. The same appendix contains runtime comparisons and additional metrics.

Across both MCAR and MNAR, TabDiff \ourname\ (ours) delivers the \textbf{best distributional metrics} and the \textbf{best average downstream rank}. Under MCAR, Figure~\ref{fig:tabular-density-overall} and Table~\ref{tab:density-overall-error} show an almost twofold reduction in distributional error w.r.t. the second best baseline across all mask rates, while Figure~\ref{fig:tabular-mle} and Table~\ref{tab:mle-ranking} give TabDiff \ourname\ the best average downstream rank from per-dataset ML-efficiency errors. Under MNAR at $p=0.7$, Figure~\ref{fig:tabular-mnar} reproduces this picture, with TabDiff \ourname\ again leading on both Overall Density and ML-efficiency.

% We also run an ablation on the quality of categorical features imputation under the same MCAR setup in see Appendix~\ref{app:categorical-only}.

% \textbf{Comparisons}: to other imputation methods such as diffputer (diffusion with EM counterpart), remasker (masked autoencoder approach), hyperimpute which i don't have at the moment (classic ML approach) and others if possible.

% \textbf{Analysis}: Fig.~\ref{fig:tabular-density-overall} and Fig.~\ref{fig:tabular-mle} report overall density quality and MLE; additional metrics (e.g.\ Trend/Shape) follow the same layout per dataset.

% \begin{itemize}
%   \item Compare Trend/Shape and MLE. Maybe make one table at some particualr missing rate (70%?, 50%?).
%   \item For each of the dataset plot the curve Missing Rate vs Metrics
%   \item Analyse the cat features distinctly. Plot trend/shape and MLE
% \end{itemize}

%% file: chapters/appendix.tex
\appendices

\vspace{-2mm}
\section{Proofs}\label{app:proofs}
\vspace{-2mm}

\begin{proof}[Proof of the Proposition \ref{prop:impute-em-update}]
Fix the current iterate $p^n \in \mathcal{P}(\mathcal{X})$. In the exact E-step, for each mask
$m \in \mathcal{M}$ the variational conditional over the missing coordinates is
\begin{equation}
    q^n(x_{\bar m}\mid x_m) = p^n(x_{\bar m}\mid x_m),
\end{equation}
which induces the completed-data measure
\begin{equation}
    \pi_m(p^n)(x) := p_m^*(x_m)\,p^n(x_{\bar m}\mid x_m).
\end{equation}
The M-step maximizes the expected complete-data log-likelihood, which in terms of the mixture
\begin{equation}
    \bar \pi^n(x)
    := \mathbb{E}_{m\sim\mu}\bigl[\pi_m(p^n)(x)\bigr]
    = \mathbb{E}_{m\sim\mu}\bigl[p_m^*(x_m)\,p^n(x_{\bar m}\mid x_m)\bigr]
\end{equation}
reads $\argmax_{p\in\mathcal P(\mathcal X)} \int_{\mathcal X} \log p(x)\, d\bar \pi^n(x)$. Since
$\KL{\bar \pi^n}{p}\ge 0$,
\begin{equation}
    \int_{\mathcal X} \log p\, d\bar \pi^n
    = \int_{\mathcal X} \log \bar \pi^n\, d\bar \pi^n - \KL{\bar \pi^n}{p}
    \le \int_{\mathcal X} \log \bar \pi^n\, d\bar \pi^n,
\end{equation}
with equality if and only if $p=\bar \pi^n$. Hence the maximizer is $p^{n+1}=\bar \pi^n$. Therefore the exact update is
\begin{equation}
    p^{n+1}(x)
    =
    \bar \pi^n(x)
    =
    \mathbb{E}_{m\sim\mu}\bigl[p_m^*(x_m)\,p^n(x_{\bar m}\mid x_m)\bigr],
\end{equation}
which is precisely the claimed Impute-EM update.
\end{proof}

\begin{proof}[Proof of the Theorem \ref{thm:lyapunov-inequality}]
Fix $m\in\mathcal M$ and let $p$ satisfy $\KL{r}{p}<\infty$ for some $r\in\mathcal C$, and set $\pi_m(p)(x):=p(\overline{x}_m\mid x_m)\,p_m^*(x_m)$. By the chain rule for relative entropy,
\begin{equation}\label{eq:KL_equality}
    \KL{r}{p}=\KL{p_m^*}{p_m}+\KL{r}{\pi_m(p)}.
\end{equation}
By convexity of relative entropy in its second argument (the log-sum inequality), for measures $s_m$ with $\KL{r}{s_m}<\infty$,
\begin{equation}\label{eq:convex_ineq}
    \KL{r}{\sum_{m\in\mathcal M}\mu(m)s_m}
    \le
    \sum_{m\in\mathcal M}\mu(m)\KL{r}{s_m}.
\end{equation}
An induction using \eqref{eq:KL_equality} and $\mu(m)>0$ keeps $r\ll p^n$ and $\KL{r}{p^n}<\infty$. Since $p^{n+1}=\sum_{m}\mu(m)\pi_m(p^n)$, applying \eqref{eq:convex_ineq} and then \eqref{eq:KL_equality} gives
\[
\begin{aligned}
    \KL{r}{p^{n+1}}
    &\le \sum_{m\in\mathcal M}\mu(m)\KL{r}{\pi_m(p^n)}\\
    &= \KL{r}{p^n} - \sum_{m\in\mathcal M}\mu(m)\KL{p_m^*}{p^n_m}.
\end{aligned}
\]
Thus
\begin{equation}\label{eq:one_step}
    \KL{r}{p^{n+1}}
    +
    \sum_{m\in\mathcal M}
    \mu(m)\KL{p_m^*}{p^n_m}
    \le\KL{r}{p^n}.
\end{equation}
Summing \eqref{eq:one_step} over $n$ and letting $N\to\infty$,
\[
    \sum_{n=0}^{\infty}
    \sum_{m\in\mathcal M}
    \mu(m)\KL{p_m^*}{p^n_m}
    \le
    \KL{r}{p^0}
    <\infty,
\]
so each nonnegative term satisfies $\KL{p_m^*}{p^n_m}\to 0$ as $n\to\infty$ for every $m\in\mathcal M$.
\end{proof}

\vspace{-2mm}
\section{Additional methodology clarifications}\label{app:add-methods}
\vspace{-2mm}

\subsection{Incomplete likelihood}\label{app:incomplete_likelihood}
\vspace{-1mm}

We use the term \emph{incomplete likelihood} for objectives that train a generative model using only observed coordinates, i.e.,
\[
    \max_\theta \;
    \mathbb{E}_{m\sim\mu,\;x_m\sim p_m^*}
    \left[\log p_{\theta,m}(x_m)\right],
\]
or tractable diffusion/variational surrogates thereof. Under this terminology, MissDiff \cite{ouyang2023missdiff} can be viewed as an incomplete-likelihood diffusion method, since its masked denoising score-matching loss is motivated as an upper bound on the negative observed-data likelihood.  ReMasker \cite{du2024remasker} can also be viewed as a likelihood-based masked autoencoder: it artificially re-masks observed entries and trains the model to predict them from the remaining observed entries. Equivalently, it optimizes a conditional reconstruction likelihood over observed coordinates, rather than a full joint likelihood over complete data.

The ability to train models via incomplete likelihood naturally comes from the our M-step formulation, see Section \ref{sec:methods-practical-impl}.

\subsection{Mixed space diffusion models}
\vspace{-1mm}
\label{app:mixed-diffusion}

Following TabDiff~\cite{tabdiff}, we combine continuous and discrete diffusion components, writing \(x^{\rm num}\) and \(x^{\rm cat}\) for the numerical and categorical coordinates. The mixed conditional diffusion loss is, writing \(\hat p_\theta\) for the conditional categorical prediction \(\hat p_\theta(\cdot \mid x^{\rm cat}_{\bar m,t}, x_m, m)\),
\begin{equation}
  \label{eq:diffusion-loss-mixed}
  \begin{aligned}
  \mathcal{L}_{\rm mixed\,diff}(\theta)
  ={}&
  \mathbb{E}_{\substack{x \sim D^n,\, m \sim \mu\\ t \sim \mathcal{U}[0,T],\, \epsilon \sim \mathcal{N}(0,I)}}
  \left[
    \left\|
    \epsilon -
    \epsilon_\theta\!\left(x^{\rm num}_{\bar m,t}, t \mid x_m, m\right)
    \right\|^2
  \right]
  \\
  &+
  \mathbb{E}_{\substack{x \sim D^n,\, m \sim \mu\\ t \sim \mathcal{U}[0,T]\\
  x^{\rm cat}_{\bar m,t} \sim q_{\rm MDM,t}(\cdot \mid x^{\rm cat}_{\bar m})}}
  \left[
    {\rm CE}\!\left(
      e_{x^{\rm cat}_{\bar m}},\,
      \hat p_\theta
    \right)
  \right],
  \end{aligned}
\end{equation}

with \(x^{\rm num}_{\bar m,t} = \sqrt{\bar\alpha_t}\,x^{\rm num}_{\bar m} + \sqrt{1-\bar\alpha_t}\,\epsilon\). The conditional reverse process factorizes into numerical and categorical transitions:
\begin{equation}
  \label{eq:diffusion-generative-mixed}
  \begin{aligned}
  &p_{\theta}\!\left(x_{\bar m,t-1} \mid x_{\bar m,t}, t, x_m, m\right)
  =
  \mathcal{N}\!\left(
    x_{\bar m,t-1}^{\rm num};
    \mu_{\theta},
    \sigma_t^2 I
  \right)
  \\
  &\quad \cdot
  {\rm Cat}\!\left(
    x_{\bar m,t-1}^{\rm cat}
    \;\middle|\;
    \frac{
      (1-\alpha_{t-1})\,\mathtt{M}
      +
      (\alpha_{t-1}-\alpha_t)\,
      \hat p_{\theta}
    }{
      1-\alpha_t
    }
  \right),
  \end{aligned}
\end{equation}
where \(\mathtt{M}\) is the categorical mask-token distribution (distinct from the family of observation masks \(\mathcal{M}\)).

\vspace{-2mm}
\section{Additional experimental results}\label{app:add-exps}
\vspace{-2mm}

\subsection{Tabular data}\label{app:tabular-add-exps}
\vspace{-1mm}

% \paragraph{Categorical features only analysis.}\label{app:categorical-only}

% This section documents the categorical-only ablation referenced from the main text. For each mixed-type dataset we remove all numerical features and keep only categorical columns, so the model operates purely in discrete state space while the training objective and masking scheme remain those of masked discrete diffusion. Missing values are still injected under MCAR at the same rates as in the main tabular experiments. We compare AmbientMDM to the same baselines and report Overall Density, defined as $(\mathrm{Trend}+\mathrm{Shape})/2$, and MLE with the same conventions as in the main paper, using the same panel layout as Figures~\ref{fig:tabular-density-overall} and~\ref{fig:tabular-mle}.

\paragraph{Runtime analysis}

Table~\ref{tab:runtime_adult_50} compares training runtimes for the tabular methods in Section \ref{sec:experiments-tabular}. Our method is slower than the incomplete-likelihood baselines MissDiff \cite{ouyang2023missdiff} and Remasker \cite{du2024remasker} but only by a modest margin, and $3$ to $4$ times faster than Diffputer \cite{diffputer}, whose E-step averages 10 imputations and pays a corresponding inference cost.

\begin{table}[h]
\centering
\small
\caption{Runtime comparison for training the imputation methods on \textit{default} dataset with 50\% missingness rate. All methods except HyperImpute were trained on an NVIDIA A100 GPU. HyperImpute utilizes only the CPU.}
\label{tab:runtime_adult_50}
\begin{tabular}{lc}
\toprule
\textbf{Method} & \textbf{Runtime} \\
\midrule
Remasker & $1$ hour \\
MissDiff & $20$ minutes \\
Diffputer & $7$ hours $18$ minutes \\
HyperImpute & \textbf{22 minutes} \\
TabDiff & $53$ minutes \\
TabDiff  \ourname\ (ours) & $2$ hours $41$ minutes \\
\bottomrule
\end{tabular}
\vspace{-4mm}
\end{table}

\paragraph{\ourname\ iterations analysis}\label{app:tabular-iterations}

Figure~\ref{fig:tabular-iter-dynamics} shows ML-efficiency, Shape, and Trend versus \ourname\ iterations on the Default dataset with $50\%$ missingness. All metrics improve with iterations, with most of the gain in the first few steps and convergence by step 5.

\begin{figure}[t]
  \centering
  \includegraphics[width=\linewidth]{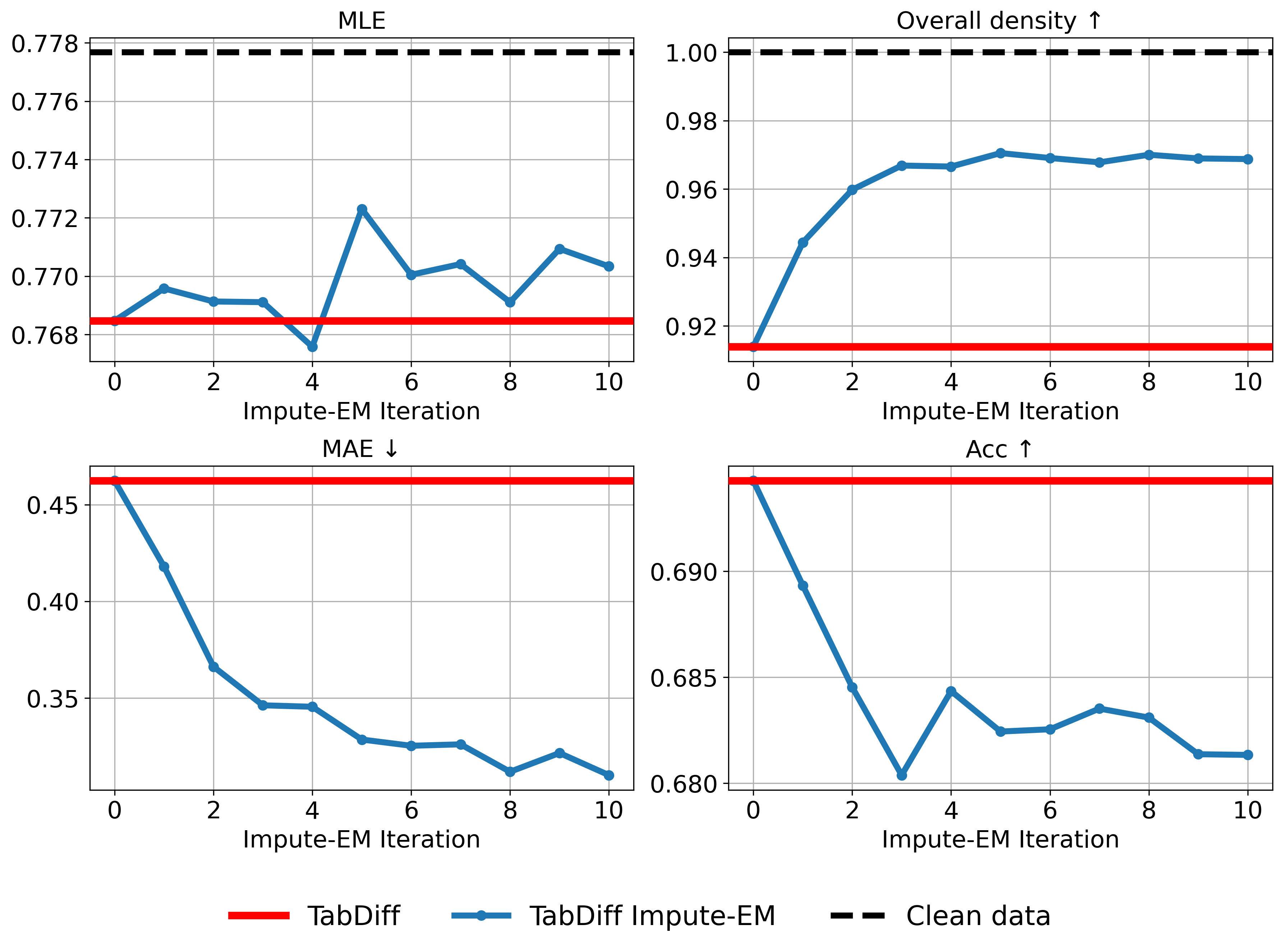}%
  \caption{Metrics versus Impute-EM iteration on the tabular data, i.e., "default" dataset with $50\%$ missingness.}
  \label{fig:tabular-iter-dynamics}
\vspace{-4mm}
\end{figure}

\paragraph{Discriminative Metrics}\label{app:discriminative-metrics}

Our objective is distributional recovery rather than pointwise reconstruction, so discriminative metrics are less relevant for our method. We nevertheless report them on imputed features under the same MCAR setup as in Section~\ref{sec:experiments-tabular}. \textbf{Accuracy} measures how well imputed categorical entries match the ground truth, and \textbf{MAE} measures mean absolute error on numerical features. Figures~\ref{fig:tabular-appendix-acc} and~\ref{fig:tabular-appendix-mae} show both metrics versus mask rate for Adult, Shoppers, Default, and Beijing. While TabDiff \ourname\ achieves the best distributional metrics and the best average downstream ML-efficiency rank, it does not lead on discriminative metrics, as the compared baselines are designed for pointwise recovery \cite{ouyang2023missdiff, du2024remasker, diffputer}. Figure~\ref{fig:tabular-mnar-discriminative} reports the same metrics under MNAR at $p=0.7$, with the same conclusion.

\begin{figure}[t]
  \centering
  \includegraphics[width=\linewidth]{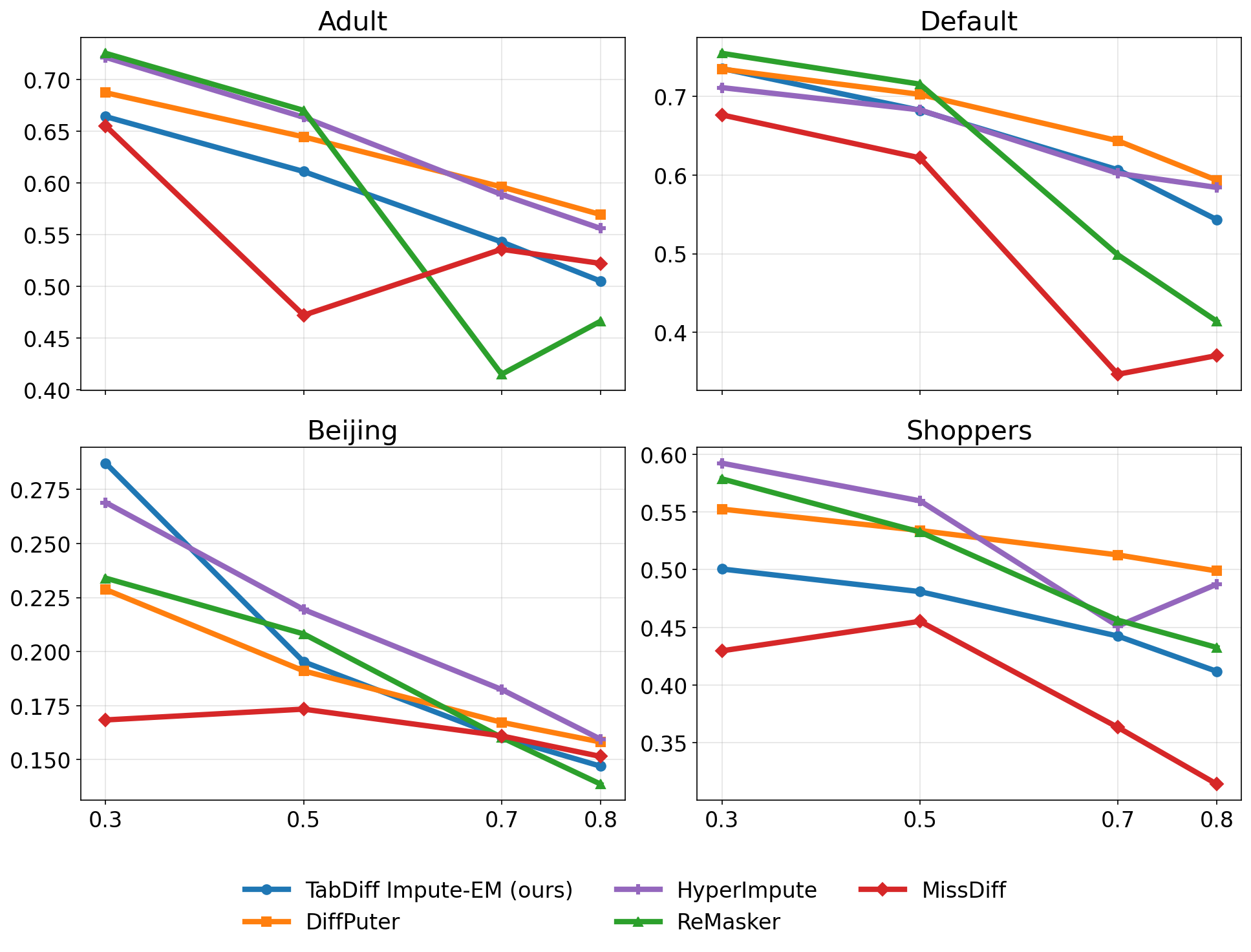}%
  \caption{Categorical accuracy $\uparrow$ under MCAR versus mask rate (Adult, Shoppers, Default, and Beijing).}
  \label{fig:tabular-appendix-acc}
\vspace{-4mm}
\end{figure}

\begin{figure}[t]
  \centering
  \includegraphics[width=\linewidth]{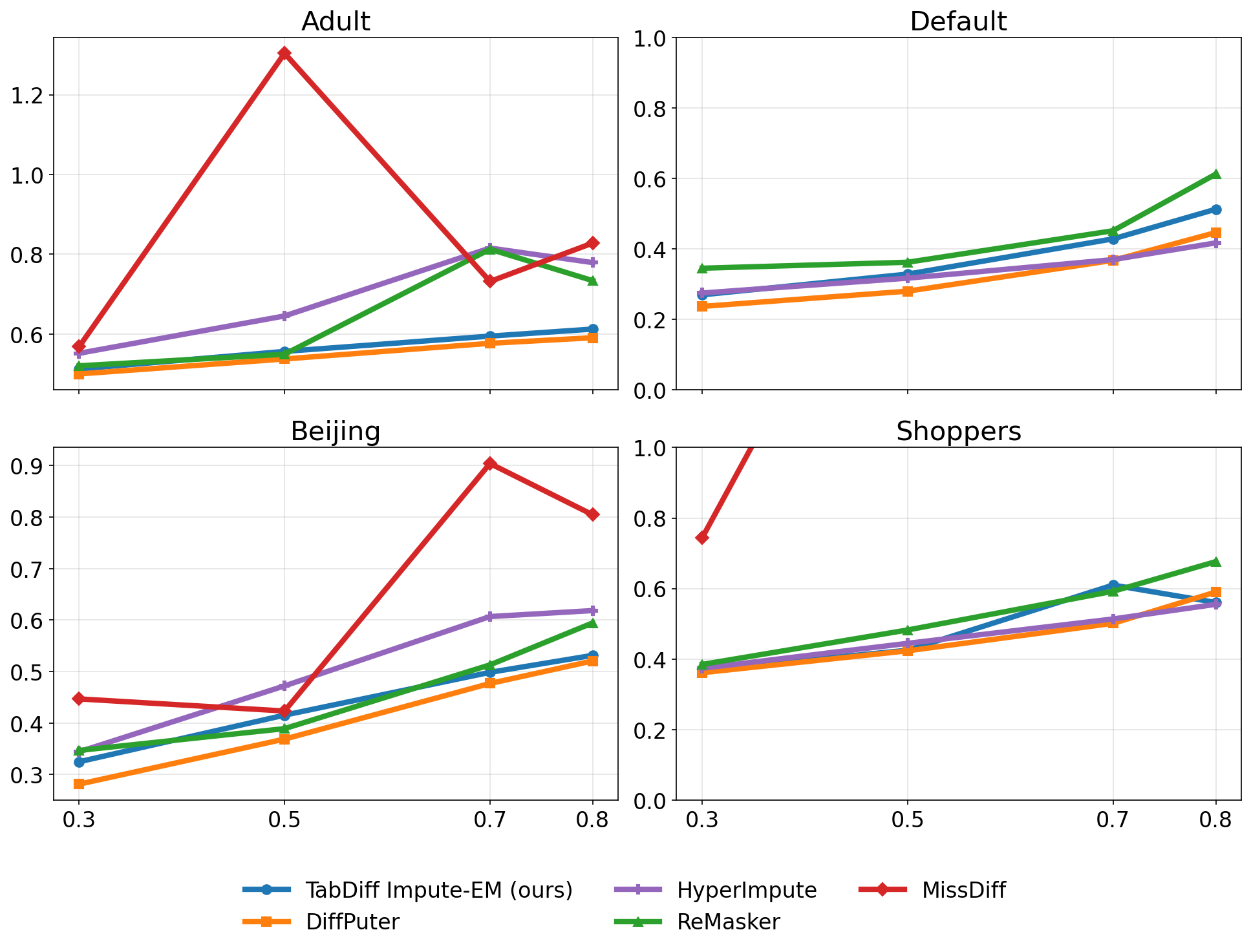}%
  \caption{Mean absolute error (MAE) $\downarrow$ on numerical features under MCAR versus mask rate (Adult, Shoppers, Default, and Beijing).}
  \label{fig:tabular-appendix-mae}
\vspace{-4mm}
\end{figure}

\begin{figure}[t]
  \centering
  \includegraphics[width=\linewidth]{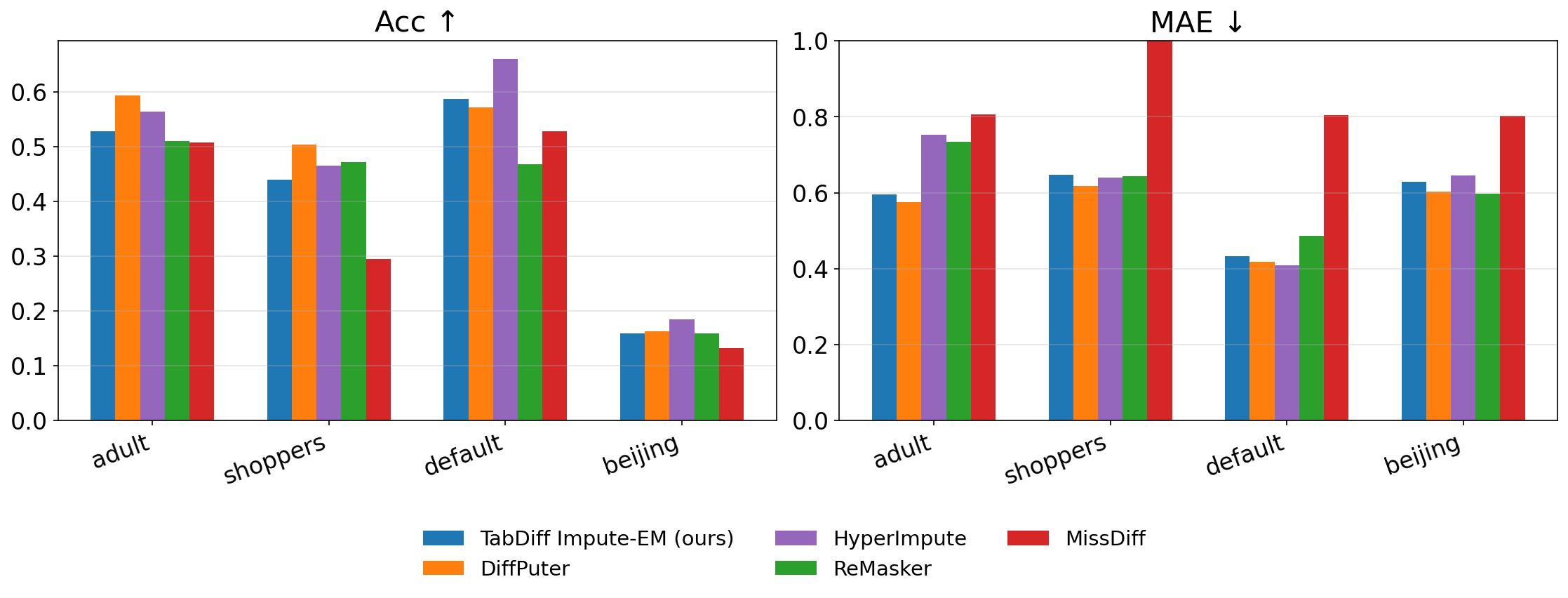}%
  \caption{Discriminative metrics under MNAR at mask rate $p=0.7$. Categorical accuracy $\uparrow$ and MAE $\downarrow$ across the four datasets.}
  \label{fig:tabular-mnar-discriminative}
\vspace{-4mm}
\end{figure}

\subsection{MAR robustness}\label{app:mar-results}
\vspace{-1mm}

We additionally evaluate MAR (Missing At Random) robustness with the DiffPuter \cite{diffputer} $70\%$ masking protocol, with Adult, Shoppers, Default, and Beijing datasets. All the methods are run once on the same test-set evaluation. TabDiff and DiffPuter use five EM iterations, while ReMasker \cite{du2024remasker} and MissDiff \cite{ouyang2023missdiff} are single-pass baselines. Figure~\ref{fig:tabular-mar-density-mle} reports distributional fidelity and downstream utility. Figure~\ref{fig:tabular-mar-discriminative} reports pointwise metrics. TabDiff \ourname\ has the highest Overall Density score on every dataset and deliveres less MLE error than the other methods overall.

\begin{figure}[t]
  \centering
  \includegraphics[width=0.9\linewidth]{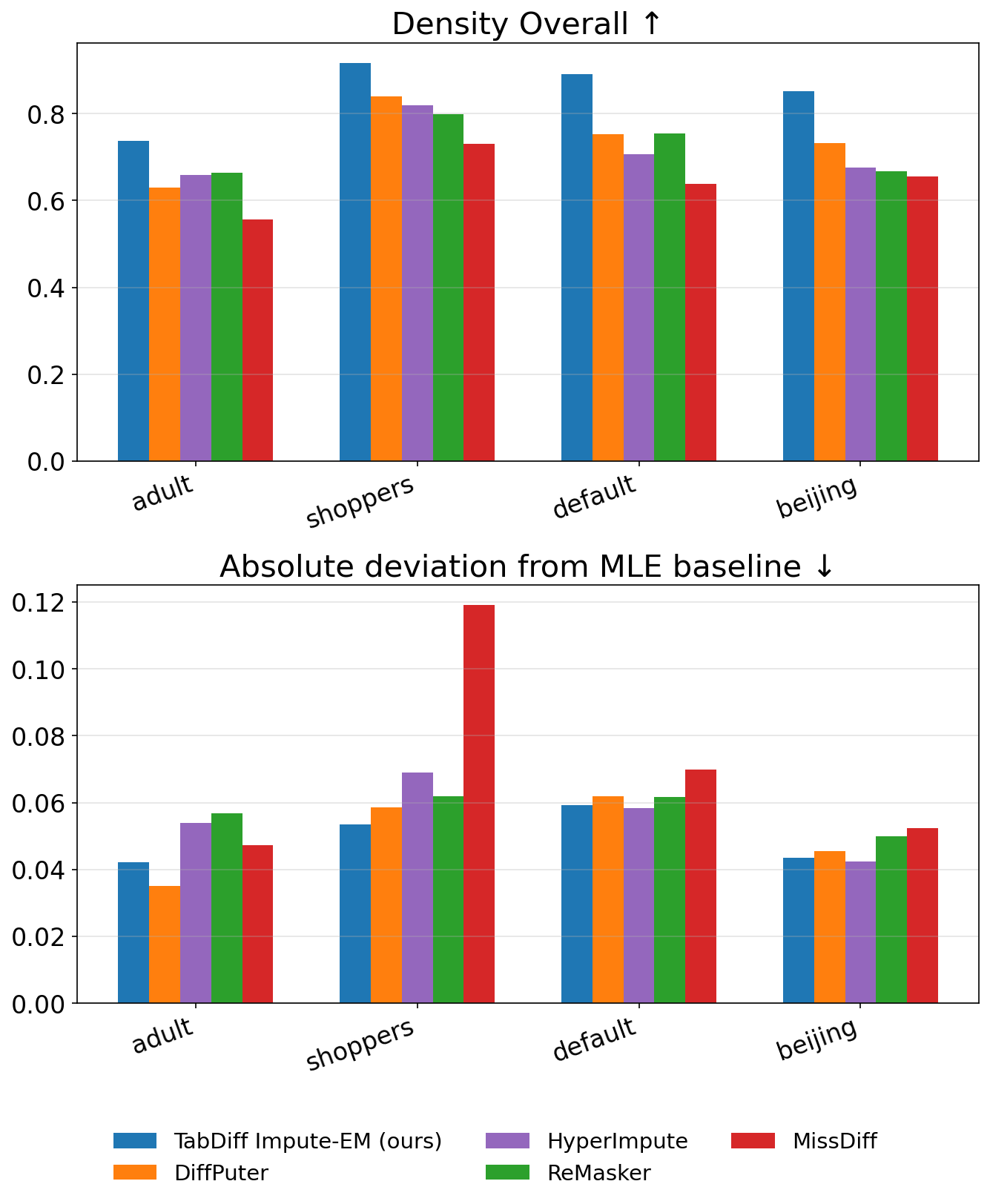}%
  \caption{Single-run MAR results under the DiffPuter $70\%$ mask, split $0$. Density Overall $\uparrow$ and absolute deviation from the clean-data MLE baseline $\downarrow$ across the four datasets. Beijing MLE deviation is multiplied by $0.1$ to keep the scale comparable across datasets.}
  \label{fig:tabular-mar-density-mle}
\vspace{-4mm}
\end{figure}

\begin{figure}[t]
  \centering
  \includegraphics[width=\linewidth]{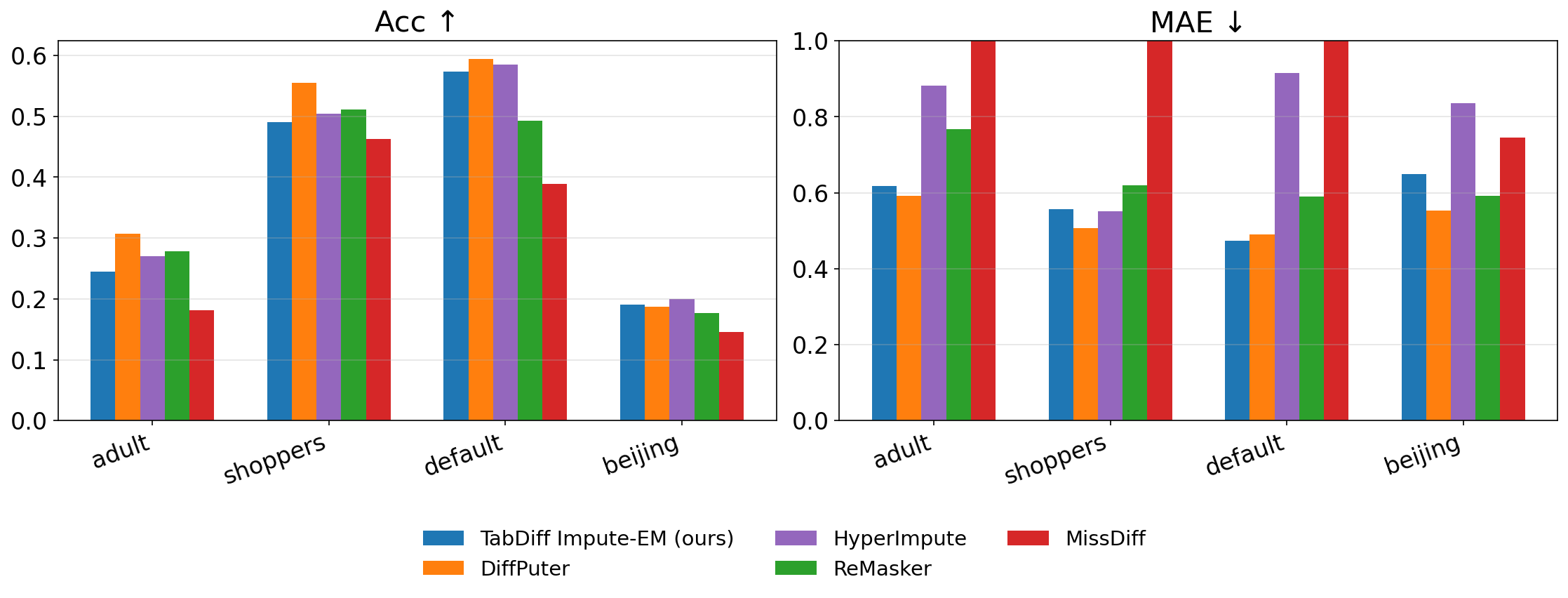}%
  \caption{Single-run MAR discriminative results under the DiffPuter $70\%$ mask, split $0$. Categorical accuracy $\uparrow$ and MAE $\downarrow$ across the four datasets. The MAE axis is capped at $1.0$.}
  \label{fig:tabular-mar-discriminative}
\vspace{-4mm}
\end{figure}

\vspace{-2mm}
\section{Experimental details}\label{app:exp-details}
\vspace{-2mm}

\subsection{Text data}\label{app:details-text}
\vspace{-1mm}

\paragraph{Metrics}\label{app:text-metrics} Perplexity (PPL) is the exponential of the average negative log-likelihood, $\mathrm{PPL} = \exp\!\left(-\mathbb{E}_{x\sim p_{\mathrm{data}}}\!\left[\tfrac{1}{N}\log p_\theta(x_1,\ldots,x_N)\right]\right)$, where the log-likelihood is computed via the MDLM ELBO \cite{sahoo2024mdlm}.

\paragraph{Implementation details} We utilize the MDLM \cite{sahoo2024mdlm} official code base:

\vspace{-2mm}
\begin{center}
\url{https://github.com/kuleshov-group/mdlm}
\end{center}
\vspace{-2mm}

The following hyperparameters are consistent across all experiments. During E-step dataset generation for MDM backward process inference, we use $64$ function evaluations (NFE). As the backbone architecture, we use a DIT \cite{peebles2023scalable} with approximately $13$M parameters. EMA of weights was used with $0.999$ coefficient. The training took $10$ hours for ReMasker MDLM and $17$ hours for MDLM \ourname\ (including the initialization stage) on two NVIDIA-A100 GPUs.

\subsection{Tabular data}\label{app:details-tabular}
\vspace{-1mm}

\paragraph{Missingness patterns}
Missing Completely At Random (MCAR) means missingness is independent of all entries. It is the standard tabular imputation benchmark for its simplicity and reproducibility. For our Missing Not At Random (MNAR) robustness study at $p=0.7$, we follow the MNAR protocol of DiffPuter \cite{diffputer}: the columns are split into two groups, a logistic model on the first group outputs the missing probabilities for the second group, and MCAR is then applied to the first group. As a result, missingness in the second group depends on the masked values of the first. For more details we refer the reader to DiffPuter \cite{diffputer}.

\paragraph{Distributional metrics (Shape and Trend)}\label{app:shape}\label{app:trend}
Following TabDiff \cite{tabdiff}, we use the SDMetrics\footnote{\url{https://docs.sdv.dev/sdmetrics}} Shape and Trend scores. \emph{Shape} measures how well each column marginal matches the reference data, using the Kolmogorov-Smirnov statistic for numerical columns and total variation distance for categorical columns. \emph{Trend} measures how well pairwise dependence matches, with Pearson correlations for numerical pairs and contingency-table frequencies for categorical pairs. \emph{Overall Density} is their average $(\mathrm{Shape}+\mathrm{Trend})/2$. Lower error is better.

\paragraph{Machine learning efficiency}\label{app:mle-metric}
Machine Learning Efficiency (ML-efficiency) evaluates whether imputed tables remain useful for the dataset's supervised task \cite{tabdiff}. We train an XGBoost model on the imputed training split (hyperparameters chosen on an $8{:}1$ train-validation split) and evaluate on a held-out clean test set, using AUC-ROC for classification and RMSE for regression. Only training features are imputed, with scores averaged over $50$ runs.

% In the generative setting of TabDiff \cite{tabdiff}, one compares metrics when training on real data versus on synthetic samples of the same size. Here we compare training on clean complete data versus training on MCAR-imputed data at each missing rate. Following the figures and tables in the main text, we report the absolute deviation of the imputed-data test score from the score obtained when training on the clean data (AUC for classification, RMSE for regression), so lower deviation is better.

\paragraph{Data description}\label{app:data-description}

We use four mixed-type tabular benchmarks from the UCI Machine Learning Repository.\footnote{\url{https://archive.ics.uci.edu/}} Each dataset is tied to a downstream supervised task used for ML-efficiency. Adult, Default, and Shoppers are classification problems. Beijing is a regression problem. In the Table~\ref{tab:app-datasets} one can see the information about datset sizes and train test splits. All preprocessing follows DiffPuter~\cite{diffputer}, whose codebase also releases the preprocessed splits.\footnote{\url{https://github.com/hengruizhang98/DiffPuter}}

\begin{table}[t]
  \centering
  \footnotesize
  \caption{Dataset statistics. \# Num and \# Cat count numerical and categorical columns and task target columns are counted separately.}
  \label{tab:app-datasets}
  \begin{tabular}{@{}lrrrrrl@{}}
    \toprule
    Dataset & \# Rows & \# Num & \# Cat & \# Train & \# Test & Task \\
    \midrule
    Adult     & 32{,}561 & 6  & 7  & 22{,}792 & 9{,}769  & Classification \\
    Default   & 30{,}000 & 14 & 9 & 21{,}000 & 9{,}000  & Classification \\
    Shoppers  & 12{,}330 & 10 & 6  & 8{,}631  & 3{,}699  & Classification \\
    Beijing   & 41{,}828 & 6  & 5  & 29{,}320 & 12{,}528  & Regression \\
    \bottomrule
  \end{tabular}
\vspace{-4mm}
\end{table}

\paragraph{Implementation details}\label{app:tabular-data-details}

Baseline methods, i.e., MissDiff, Remasker, Diffputer,  were taken from Diffputer \cite{diffputer} official repository:

\vspace{-2mm}
\begin{center}
\url{https://github.com/hengruizhang98/DiffPuter}
\end{center}
\vspace{-2mm}

All the baseline methods hyperparameters were taken from Diffuputer \cite{diffputer}. The MissDiff and Remasker experiments took no longer than hour, Diffputer experiments took no longer than $12$ hours. Models were trained on NVIDIA-A100 GPU.

As the baseline implementation for TabDiff \ourname\ we used the official TabDiff \cite{tabdiff} github repository:

\vspace{-2mm}
\begin{center}
\url{https://github.com/minkaixu/tabdiff}
\end{center}
\vspace{-2mm}

As the backbone neural network the same multimodal MLP as in TabDiff was used. At the initialization stage the incomplete likelihood model was trained for $5000$ epochs, while during each of the \ourname\ iterations the model was trained for $2000$ epochs. During the E-step the samples were generated by TabDiff with $50$ NFE. The model was trained with Adam optimizer, starting learning rate is $3e-4$ with ReduceOnPlateou scheduler. EMA of weights was used with $0.997$ coefficient. Each experiment took no longer than $12$ hours on NVIDIA-A100 GPU.

%% file: references.bib
@inproceedings{du2024remasker,
  title={Remasker: Imputing tabular data with masked autoencoding},
  author={Du, Tianyu and Melis, Luca and Wang, Ting},
  booktitle={International Conference on Learning Representations},
  volume={2024},
  pages={9002--9024},
  year={2024}
}

@article{dror2025token,
  title={Token-based Audio Inpainting via Discrete Diffusion},
  author={Dror, Tali and Shoham, Iftach and Buchris, Moshe and Gal, Oren and Permuter, Haim and Katz, Gilad and Nachmani, Eliya},
  journal={arXiv preprint arXiv:2507.08333},
  year={2025}
}

@inproceedings{you2020handling,
 author = {You, Jiaxuan and Ma, Xiaobai and Ding, Yi and Kochenderfer, Mykel J and Leskovec, Jure},
 booktitle = {Advances in Neural Information Processing Systems},
 editor = {H. Larochelle and M. Ranzato and R. Hadsell and M.F. Balcan and H. Lin},
 pages = {19075--19087},
 publisher = {Curran Associates, Inc.},
 title = {Handling Missing Data with Graph Representation Learning},
 url = {https://proceedings.neurips.cc/paper_files/paper/2020/file/dc36f18a9a0a776671d4879cae69b551-Paper.pdf},
 volume = {33},
 year = {2020}
}

@inproceedings{pujianto2019knn,
 author={Murti, Della Murbarani Prawidya and Pujianto, Utomo and Wibawa, Aji Prasetya and Akbar, Muhammad Iqbal},
  booktitle={2019 5th International Conference on Science in Information Technology (ICSITech)}, 
  title={K-Nearest Neighbor (K-NN) based Missing Data Imputation}, 
  year={2019},
  volume={},
  number={},
  pages={83-88},
  doi={10.1109/ICSITech46713.2019.8987530}}

@article{garcia2010pattern,
  title={Pattern classification with missing data: a review},
  author={Garc{\'\i}a-Laencina, Pedro J and Sancho-G{\'o}mez, Jos{\'e}-Luis and Figueiras-Vidal, An{\'\i}bal R},
  journal={Neural Computing and Applications},
  volume={19},
  number={2},
  pages={263--282},
  year={2010},
  publisher={Springer}
}

@inproceedings{zhong2023igrm,
  title={Data imputation with iterative graph reconstruction},
  author={Zhong, Jiajun and Gui, Ning and Ye, Weiwei},
  booktitle={Proceedings of the AAAI conference on artificial intelligence},
  volume={37},
  number={9},
  pages={11399--11407},
  year={2023}
}

@inproceedings{kim2026augmask,
  title={AugMask: Training Diffusion Models on Incomplete Tabular Data via Stochastic Augmentation and Masking},
  author={Kim, Jungkyu and Park, Taeyoung and Lee, Kibok},
  booktitle={Proceedings of the 43rd International Conference on Machine Learning},
  series={Proceedings of Machine Learning Research},
  volume={306},
  year={2026},
  publisher={PMLR},
  url={https://arxiv.org/abs/2606.03347}
}

@InProceedings{yoon2018gain,
  title = 	 {{GAIN}: Missing Data Imputation using Generative Adversarial Nets},
  author =       {Yoon, Jinsung and Jordon, James and van der Schaar, Mihaela},
  booktitle = 	 {Proceedings of the 35th International Conference on Machine Learning},
  pages = 	 {5689--5698},
  year = 	 {2018},
  editor = 	 {Dy, Jennifer and Krause, Andreas},
  volume = 	 {80},
  series = 	 {Proceedings of Machine Learning Research},
  month = 	 {10--15 Jul},
  publisher =    {PMLR},
  url = 	 {https://proceedings.mlr.press/v80/yoon18a.html}
}

@inproceedings{mattei2019miwae,
  title={MIWAE: Deep generative modelling and imputation of incomplete data sets},
  author={Mattei, Pierre-Alexandre and Frellsen, Jes},
  booktitle={International conference on machine learning},
  pages={4413--4423},
  year={2019},
  organization={PMLR}
}

@inproceedings{richardson2020mcflow,
  title={Mcflow: Monte carlo flow models for data imputation},
  author={Richardson, Trevor W and Wu, Wencheng and Lin, Lei and Xu, Beilei and Bernal, Edgar A},
  booktitle={Proceedings of the IEEE/CVF conference on computer vision and pattern recognition},
  pages={14205--14214},
  year={2020}
}

@article{ouyang2023missdiff,
  title={Missdiff: Training diffusion models on tabular data with missing values},
  author={Ouyang, Yidong and Xie, Liyan and Li, Chongxuan and Cheng, Guang},
  journal={arXiv preprint arXiv:2307.00467},
  year={2023}
}

@inproceedings{zheng2022tabcsdi,
  title={Diffusion models for missing value imputation in tabular data},
  author={Zheng, Shuhan and Charoenphakdee, Nontawat},
  booktitle={NeurIPS 2022 First Table Representation Workshop}
}

@article{ho2020denoising,
  title={Denoising diffusion probabilistic models},
  author={Ho, Jonathan and Jain, Ajay and Abbeel, Pieter},
  journal={Advances in neural information processing systems},
  volume={33},
  pages={6840--6851},
  year={2020}
}

@inproceedings{diffputer,
  title={Diffputer: Empowering diffusion models for missing data imputation},
  author={Zhang, Hengrui and Fang, Liancheng and Wu, Qitian and Yu, Philip},
  booktitle={International Conference on Learning Representations},
  volume={2025},
  pages={63164--63185},
  year={2025}
}

@article{dempster1977em,
  title={Maximum likelihood from incomplete data via the EM algorithm},
  author={Dempster, Arthur P and Laird, Nan M and Rubin, Donald B},
  journal={Journal of the royal statistical society: series B (methodological)},
  volume={39},
  number={1},
  pages={1--22},
  year={1977},
  publisher={Wiley Online Library}
}

@article{yu2025missing_miri,
  title={Missing data imputation by reducing mutual information with rectified flows},
  author={Yu, Jiahao and Ying, Qizhen and Wang, Leyang and Jiang, Ziyue and Liu, Song},
  journal={Advances in Neural Information Processing Systems},
  volume={38},
  pages={80324--80352},
  year={2025}
}

@article{sahoo2024mdlm,
  title={Simple and effective masked diffusion language models},
  author={Sahoo, Subham S and Arriola, Marianne and Schiff, Yair and Gokaslan, Aaron and Marroquin, Edgar and Chiu, Justin T and Rush, Alexander and Kuleshov, Volodymyr},
  journal={Advances in Neural Information Processing Systems},
  volume={37},
  pages={130136--130184},
  year={2024}
}

@inproceedings{
song2021sde,
title={Score-Based Generative Modeling through Stochastic Differential Equations},
author={Yang Song and Jascha Sohl-Dickstein and Diederik P Kingma and Abhishek Kumar and Stefano Ermon and Ben Poole},
booktitle={International Conference on Learning Representations},
year={2021},
url={https://openreview.net/forum?id=PxTIG12RRHS}
}

@article{austin2021structured,
  title={Structured denoising diffusion models in discrete state-spaces},
  author={Austin, Jacob and Johnson, Daniel D and Ho, Jonathan and Tarlow, Daniel and Van Den Berg, Rianne},
  journal={Advances in neural information processing systems},
  volume={34},
  pages={17981--17993},
  year={2021}
}

@inproceedings{tabdiff,
  title={Tabdiff: a mixed-type diffusion model for tabular data generation},
  author={Shi, Juntong and Xu, Minkai and Hua, Harper and Zhang, Hengrui and Ermon, Stefano and Leskovec, Jure},
  booktitle={International Conference on Learning Representations},
  volume={2025},
  pages={37353--37375},
  year={2025}
}

@article{hosseintabar2025diffem,
  title={Diffem: Learning from corrupted data with diffusion models via expectation maximization},
  author={Hosseintabar, Danial and Chen, Fan and Daras, Giannis and Torralba, Antonio and Daskalakis, Constantinos},
  journal={arXiv preprint arXiv:2510.12691},
  year={2025}
}

@inproceedings{hyperimpute,
  title={Hyperimpute: Generalized iterative imputation with automatic model selection},
  author={Jarrett, Daniel and Cebere, Bogdan C and Liu, Tennison and Curth, Alicia and van der Schaar, Mihaela},
  booktitle={International Conference on Machine Learning},
  pages={9916--9937},
  year={2022},
  organization={PMLR}
}

@inproceedings{fortuin2020gpvae,
  title={Gp-vae: Deep probabilistic time series imputation},
  author={Fortuin, Vincent and Baranchuk, Dmitry and R{\"a}tsch, Gunnar and Mandt, Stephan},
  booktitle={International conference on artificial intelligence and statistics},
  pages={1651--1661},
  year={2020},
  organization={PMLR}
}

@inproceedings{ou2025absorbing,
  title={Your absorbing discrete diffusion secretly models the conditional distributions of clean data},
  author={Ou, Jingyang and Nie, Shen and Xue, Kaiwen and Zhu, Fengqi and Sun, Jiacheng and Li, Zhenguo and Li, Chongxuan},
  booktitle={International Conference on Learning Representations},
  volume={2025},
  pages={64972--65009},
  year={2025}
}

@misc{mahoney2006text8,
  title={Large text compression benchmark},
  author={Mahoney, Matt},
  year={2011}
}

@inproceedings{peebles2023scalable,
  title={Scalable diffusion models with transformers},
  author={Peebles, William and Xie, Saining},
  booktitle={Proceedings of the IEEE/CVF international conference on computer vision},
  pages={4195--4205},
  year={2023}
}
